\documentclass[journal]{IEEEtran}
\IEEEoverridecommandlockouts

\usepackage{amsmath,amssymb,amsfonts}
\usepackage{algorithmic}
\usepackage{graphicx}
\usepackage{textcomp}
\usepackage{hyperref}
\usepackage{adjustbox}
\usepackage{makecell}

\usepackage{epsfig, psfrag, epstopdf, subeqnarray,nccfoots,color,multirow,bm,threeparttable,datetime2}
\usepackage[linesnumbered,ruled]{algorithm2e}
\usepackage{amsthm}
\usepackage{booktabs}
\usepackage{cuted}
\usepackage{cite}
\usepackage{enumerate}
\usepackage[caption=false,font=footnotesize,labelfont=rm,textfont=rm]{subfig}
\usepackage{orcidlink}

\allowdisplaybreaks

\def\BibTeX{{\rm B\kern-.05em{\sc i\kern-.025em b}\kern-.08em
    T\kern-.1667em\lower.7ex\hbox{E}\kern-.125emX}}

\newtheorem{example}{Example}
\newtheorem{theorem}{Theorem}
\newtheorem{lemma}{Lemma}
\newtheorem{remark}{Remark}

\newtheorem{assumption}{Assumption}
 
\definecolor{orange}{RGB}{200,0,100}

\def\bzeta{{\bm{\zeta}}}
\def\btheta{{{\bm{\theta}}}} 
\def\blambda{{{\bm{\lambda}}}}

\allowdisplaybreaks
 
\begin{document}
\title{
Sparsity-Aware Distributed Learning for Gaussian Processes with Linear Multiple Kernel
}

\author{
Richard Cornelius Suwandi$^{\orcidlink{0009-0001-7894-1674}}$, \IEEEmembership{Graduate Student Member, IEEE},
Zhidi Lin$^{\orcidlink{0000-0002-6673-511X}}$, 
\IEEEmembership{Graduate Student Member, IEEE}, \\ Feng Yin$^{\orcidlink{0000-0001-5754-9246}}$, \IEEEmembership{Senior Member, IEEE}, Zhiguo Wang$^{\orcidlink{0000-0003-2604-3597}}$, \IEEEmembership{Member, IEEE}, 
and Sergios Theodoridis$^{\orcidlink{0000-0001-5040-161X}}$, \IEEEmembership{Life Fellow, IEEE}
\thanks{This work was supported in part by the NSFC under Grant 62271433 and in part by Shenzhen Science and Technology Program under Grant JCYJ20220530143806016 and  RCJC20210609104448114. \textit{(Corresponding author: Feng
Yin.)}}
\thanks{This paper is an extension of our prior work \cite{suwandiGaussianProcessRegression2022} presented at the IEEE International Conference on  Information Fusion (FUSION) in 2022.}
\thanks{R.C. Suwandi and F. Yin are with the School of Science and Engineering, The Chinese University of Hong Kong, Shenzhen 518172, China  (e-mail: \href{mailto:richardsuwandi@link.cuhk.edu.cn}{richardsuwandi@link.cuhk.edu.cn}, \href{mailto:yinfeng@cuhk.edu.cn}{yinfeng@cuhk.edu.cn}).}
\thanks{
Z. Lin is with the Department of Statistics and Data Science, National University of Singapore, Singapore 117546, and was with the School of Science and Engineering, The Chinese University of Hong Kong, Shenzhen 518172, China (email: \href{mailto:zhidilin@link.cuhk.edu.cn}{zhidilin@nus.edu.sg}).}
\thanks{Z. Wang is with the Department of Mathematics, Sichuan University, Chengdu 610064, Sichuan, China (email: \href{mailto:wangzhiguo@scu.edu.cn}{wangzhiguo@scu.edu.cn}).}
\thanks{S. Theodoridis is with the Department of Informatics and Telecommunications of the National and Kapodistrian University of Athens (email: \href{mailto:stheodor@di.uoa.gr}{stheodor@di.uoa.gr}).}
\thanks{R.C. Suwandi and Z. Lin contribute equally and are the co-first authors.} 
}

\maketitle

\begin{abstract}
Gaussian processes (GPs) stand as crucial tools in machine learning and signal processing, with their effectiveness hinging on kernel design and hyper-parameter optimization. This paper presents a novel GP linear multiple kernel (LMK) and a generic sparsity-aware distributed learning framework to optimize the hyper-parameters. The newly proposed grid spectral mixture product (GSMP) kernel is tailored for multi-dimensional data, effectively reducing the number of hyper-parameters while maintaining good approximation capability. We further demonstrate that the associated hyper-parameter optimization of this kernel yields sparse solutions. To exploit the inherent sparsity of the solutions, we introduce the Sparse LInear Multiple Kernel Learning (SLIM-KL) framework. The framework incorporates a quantized alternating direction method of multipliers (ADMM) scheme for collaborative learning among multiple agents, where the local optimization problem is solved using a distributed successive convex approximation (DSCA) algorithm. SLIM-KL effectively manages large-scale hyper-parameter optimization for the proposed kernel, simultaneously ensuring data privacy and minimizing communication costs. Theoretical analysis establishes convergence guarantees for the learning framework, while experiments on diverse datasets demonstrate the superior prediction performance and efficiency of our proposed methods.
\end{abstract}


\section{Introduction}
\label{sec:Introduction}
\IEEEPARstart{R}{ecently}, Bayesian methods have gained significant attention across a wide range of machine learning and signal processing applications. They are known for their ability to enhance model robustness by leveraging prior knowledge and effectively handling uncertainties associated with the estimates \cite{chengRethinkingBayesianLearning2022}. Among these methods, Gaussian process (GP) has emerged as a particularly promising model that offers reliable estimates accompanied with uncertainty quantifications \cite{rasmussenGaussianProcessesMachine2006, pan2017prediction, schmidt2023probabilistic}. The distinctive strength of GP lies in its explicit modeling of the underlying function using a GP prior, which empowers GP to capture intricate patterns in the data, even with limited observations \cite{skolidis2013semisupervised, lin2023TAG}. Unlike conventional machine learning models that rely on point estimates, GP embraces a probabilistic framework that yields a posterior distribution over the function given the observed data \cite{theodoridisMachineLearningBayesian2020, chowdhary2015BayesianNonparametric}. 

The representation power of a GP model heavily relies on the selection of an appropriate kernel function. Several traditional kernel functions have been used in GP models, including the squared-exponential (SE) kernel, Ornstein-Uhlenbeck kernel, rational quadratic kernel, and periodic kernel, among others \cite{rasmussenGaussianProcessesMachine2006}. These kernel functions can even be combined, for example, into a linear multiple kernel (LMK) in the form of a linearly-weighted sum to enhance the overall modeling capacity \cite{xuWirelessTrafficPrediction2019}. However, the selection of these preliminary kernel functions often relies heavily on subjective expert knowledge, making it impractical for complex applications \cite{zhang2021kernel}.

To address the challenge of manual kernel design, L{\'a}zaro-Gredilla \textit{et al.} introduced the sparse spectrum GP (SSGP) that employs a sparse set of spectral points to approximate the spectral density of a stationary kernel function \cite{lazaro2010sparse}. However, optimizing these spectral points can lead to overfitting, especially as their number increases, resulting in poor generalization on unseen data \cite{cui2024probRL}. Wilson \textit{et al.} proposed the spectral mixture (SM) kernel to discover underlying data patterns automatically via formulating the spectral density of kernel as a Gaussian mixture \cite{wilsonGaussianProcessKernels2013}. However, the original SM kernel presents difficulties in tuning its hyper-parameters due to the non-convex nature of the optimization problem. This increases the risk of getting stuck at a sub-optimal solution, especially when dealing with a large number of hyper-parameters. The grid spectral mixture (GSM) kernel was proposed to address this drawback in the SM kernel \cite{yinLinearMultipleLowRank2020}. Specifically, the GSM kernel focuses on optimizing the weights of each kernel component while fixing other hyper-parameters, resulting to a more favorable optimization structure. Consequently, the GSM kernel acts as a LMK, allowing its sub-kernels to exhibit beneficial low-rank properties under reasonable conditions \cite{yinLinearMultipleLowRank2020}. By identifying a sparse subset of the most relevant frequency components in the input data, the GSM kernel enhances the model interpretability and facilitates a clearer understanding of the underlying data patterns.

While employing GSM kernel-based GP (GSMGP) models holds great potential \cite{chen2022recent}, it also comes with several challenges. Existing formulation of the GSM kernel is only suitable for handling one-dimensional data \cite{yinLinearMultipleLowRank2020}, and naively extending it to multi-dimensional data leads to potential issues, as the number of hyper-parameters and the model complexity grow exponentially with the dimensions \cite{suwandiGaussianProcessRegression2022}.
Additionally, centralized optimization of the GSMGP models in real-world scenarios introduces further challenges. On one hand, the $\mathcal{O}(n^3)$ computational complexity, where $n$ is the training data size, hinders the scalability of GSMGP in big data applications \cite{Liu2020scalableGP, zhai2023distributed, dang2024GP}. On the other hand, collecting data centrally becomes difficult in practice, as data are often distributed across locations or devices, and concerns about data privacy often discourage individual data owners from sharing their data \cite{yinFedLocFederatedLearning2020}.

In this paper, we address the challenge of constructing an optimal GP kernel for multi-dimensional and large-scale data. Unlike existing solutions, such as sparse variational GPs with inducing points \cite{titsias2009variational, liu2021gp}, which remain inefficient due to model complexity and data size, we propose a novel and practical learning framework that leverages advancements in distributed computing to efficiently optimize kernel hyper-parameters. Our main contributions are summarized as follows:
\begin{itemize}
\item We introduce a novel formulation for the GSM kernel capable of accommodating multi-dimensional data. Compared to existing formulations \cite{yinLinearMultipleLowRank2020,suwandiGaussianProcessRegression2022}, our approach effectively reduces the number of hyper-parameters while maintaining a good approximation capability. Furthermore, the proposed kernel exhibits a sparsity-promoting property that enables efficient optimization of the hyper-parameters. 
\item We leverage advances in distributed computing and exploit the inherent sparsity of our proposed kernel to develop a practical learning framework for GPs with LMK. The framework, termed Sparse LInear Multiple Kernel Learning (SLIM-KL), consists of two main components: 1) a quantized alternating direction method of multipliers (ADMM) scheme, allowing multiple agents to collaboratively learn the kernel hyper-parameters while preserving data privacy and reducing communication costs; and 2) a distributed successive convex approximation (DSCA) algorithm for solving the local hyper-parameter optimization problem in a distributed manner within the quantized ADMM. As a result, SLIM-KL offers scalability, maintains privacy, and ensures efficient communication.
\item We provide a comprehensive theoretical analysis for the proposed learning framework. Additionally, we conduct extensive experiments on diverse one-dimensional and multi-dimensional datasets to evaluate its performance. The results demonstrate that our proposed framework consistently outperform the existing approaches in terms of prediction performance and computation efficiency.
\end{itemize}
The remainder of this paper is organized as follows. Section \ref{sec:preliminaries} reviews the GP regression and GSM kernel. Sections \ref{sec:GSM-MISO} and \ref{sec:SLIM-KL} introduce the proposed kernel formulation and learning framework. Section \ref{sec:analysis} presents the theoretical analysis of our proposed framework. Section \ref{sec:experiments} demonstrates the experimental results. Finally, Section \ref{sec:conclusion} concludes the paper. Key technical proofs and derivations are given in the Appendix.

\section{Preliminaries}
\label{sec:preliminaries}
This section first reviews the salient GP regression model. Then, the GSM kernel and the associated hyper-parameter optimization problem are introduced.

\subsection{Gaussian Process Regression}
\label{subsec:GPR}
 
Gaussian process regression (GPR) is a powerful, Bayesian non-parametric approach in machine learning that incorporates prior knowledge and offers predictions with uncertainty estimates \cite{rasmussenGaussianProcessesMachine2006}. 
Given an observed dataset  $\mathcal{D}\triangleq\{\bm{x}_i, {y}_i\}_{i = 1}^{n} \triangleq \{\bm{X}, \bm{y}\}$ consisting of $n$ input-output pairs, let us consider the following regression model,
\begin{equation}
y = f(\bm{x}) + {e},  \quad 	{e} \sim \mathcal{N}(0, \sigma_{e}^2),
\label{eq:reg_model}
\end{equation}
where ${y}$ is the output, $e$ is the additive zero-mean Gaussian noise with variance $\sigma_e^2$, and $f(\bm{x})$ is the unknown function we aim to learn. At the heart of GPR lies the assumption that $f(\bm{x})$ follows a Gaussian process, defined as a collection of random variables wherein any finite subset follows a joint Gaussian distribution \cite{rasmussenGaussianProcessesMachine2006}. Mathematically, a GP is expressed as,
\begin{equation}
f(\bm{x}) \sim \mathcal{GP}\left(m(\bm{x}), k(\bm{x}, \bm{x}^\prime;  \bm{\theta})\right),
\end{equation}
where $m(\bm{x})$ is the mean function, typically set to zero in practice, and $k(\bm{x}, \bm{x}'; \bm{\theta})$ is the covariance function or kernel, with $\bm{\theta}$ being its hyperparameters. This kernel function encapsulates our assumptions about the function, such as smoothness and periodicity \cite{rasmussenGaussianProcessesMachine2006}.  

In the GPR setting, we are interested in predicting new function values $\bm{f}_* \!\triangleq\! f(\bm{X}_*)$ for any unseen inputs $\bm{X}_*$. Due to the Gaussian nature, the joint distribution of $\bm{y}$ and $\bm{f}_*$ can be derived as,
\begin{align}
    \left[\begin{array}{c}
        \bm{y}  \\
        \bm{f}_* 
    \end{array}\right]
    \sim
    \mathcal{N}\left(\bm{0}, 
    \left[ \begin{array}{cc}
       \boldsymbol{K}_{XX}+ \sigma_{e}^{2} \boldsymbol{I}_{n}  & \boldsymbol{K}_{ {X}_*X}^\top \\
       \boldsymbol{K}_{ {X}_*X}  & \boldsymbol{K}_{{X}_* {X}_*}
    \end{array} 
    \right] \right).
\end{align}
Here, $\bm{K}_{XX}$ and $\bm{K}_{X_* X_*}$ represent the covariance matrices derived from evaluating the kernel function $k(\bm{x}_i, \bm{x}_j; \bm{\theta})$ on the training and test inputs, respectively. Similarly, $\boldsymbol{K}_{ {X}_*X}$ denotes the cross-covariance matrix evaluated using the training and test inputs. By applying the results of conditional Gaussian distribution, the predictive distribution $p(\bm{f}_{*} \mid \mathcal{D}, \boldsymbol{X}_{*}) = \mathcal{N} \left(\bm{f}_{*} \mid  {\boldsymbol{\xi}_*} , \ {\boldsymbol{\Xi}_*}  \right)$ can be computed in closed-form, where the corresponding mean and covariance are respectively defined as 
\begin{subequations}
\setlength{\abovedisplayskip}{4pt}
\setlength{\belowdisplayskip}{4pt}
    \begin{align}
    {\boldsymbol{\xi}_*} & = \boldsymbol{K}_{{X}_* X} \left(\boldsymbol{K}_{XX}+ \sigma_{e}^{2} \boldsymbol{I}_{n} \right)^{-1} { \boldsymbol{y}}, \label{eq:post_mean}\\
    {\boldsymbol{\Xi}_*}  & = \boldsymbol{K}_{{X}_* {X}_*} - \boldsymbol{K}_{{X}_* X} \left( \boldsymbol{K}_{XX}+ \sigma_{e}^{2} \boldsymbol{I}_{n}\right)^{-1} \boldsymbol{K}_{ {X}_*X}^\top.\label{eq:post_cov}
    \end{align}
\end{subequations}
This predictive distribution shows that the GPR not only offers a point estimate for the predicted function values (the posterior mean $\boldsymbol{\xi}_*$), but also quantifies the uncertainty related to these predictions (the posterior covariance $\boldsymbol{\Xi}^*$).

\subsection{Grid Spectral Mixture Kernel}
\label{subsec:GSM}
Selecting an appropriate kernel is crucial for ensuring good performance of GP models. However, this process is often challenging and traditionally relies on subjective expert knowledge. Grid spectral mixture (GSM) kernel is a universal kernel that eliminates the need for manual kernel design and was originally proposed in \cite{yinLinearMultipleLowRank2020} to address the model training issues in the spectral mixture (SM) kernel \cite{wilsonGaussianProcessKernels2013}. Both the SM and GSM kernels leverage the fact that any stationary kernel and its spectral density are Fourier duals \cite{rasmussenGaussianProcessesMachine2006}. Specifically, the GSM kernel is designed by approximating the spectral density of the underlying kernel function in the frequency domain with a Gaussian mixture, $S({\omega})$, i.e.,
\begin{equation}
\setlength{\abovedisplayskip}{4pt}
\setlength{\belowdisplayskip}{4pt}
S({\omega}) = \frac{1}{2} \sum_{q=1}^{Q}  {\theta_q} \left[ \mathcal{N}\left({\omega} \mid {\mu}_q, \ v_q\right) + \mathcal{N}\left({\omega} \mid -{\mu}_q, \ v_q\right) \right], 
\label{eq:GMapproximatedSPD}
\end{equation}
where $Q$ is the number of Gaussian components and $\theta_q$ is the corresponding weight of the $q$-th component in the Gaussian mixture. Unlike the SM kernel, the mean variables $\{\mu_q\}_{q = 1}^Q$ and the variance variables $\{v_q\}_{q = 1}^Q$ in the GSM kernel are fixed, either uniformly or randomly, to preselected grid points \cite{yinLinearMultipleLowRank2020}. By taking the inverse Fourier transform of $S(\omega)$, we obtain the formulation for the GSM kernel as
\begin{equation}
 k(\tau) = \sum_{q=1}^{Q} {\theta_q} \underbrace{ \cos(2 \pi \tau \textcolor{black}{\mu_q}) \exp \left[ -2 \pi^{2} \tau^2 \textcolor{black}{v_{q}} \right] }_{k_{q}(\tau)},  
\label{eq:GSM-Kernel}
\end{equation}
where ${\tau}= |{x} -{x}^\prime| \in \mathbb{R}$. 
In \cite{yinLinearMultipleLowRank2020}, it has been proven that the GSM kernel is able to approximate any stationary kernel arbitrarily well. In addition, the GSM kernel formulated above also demonstrates itself to be an LMK, where $\{k_q(\tau)\}_{q=1}^Q$ are the $Q$ sub-kernels with weights $\{\theta_q\}_{q=1}^Q$ to be optimized. As an LMK, the GSM sub-kernels benefit from the low-rank property under reasonable conditions \cite{yinLinearMultipleLowRank2020}. Consequently, the GSM kernel combines the conciseness of LMKs with the representational power of the SM kernel.

One limitation of the original GSM kernel formulated in Eq.~(\ref{eq:GSM-Kernel}) is that, it is designed solely for handling one-dimensional data, i.e., data with an input dimension $P$ equal to 1. Naively extending it to multi-dimensional data, where $P > 1$, by fixing grid points in $\mathbb{R}^P$, can lead to potential issues \cite{suwandiGaussianProcessRegression2022}. Most notably, the number of grid points grows exponentially with $P$ due to the curse of dimensionality. In other words, substantially more Gaussian components are needed to approximate the spectral density in the $P$-dimensional case compared to the one-dimensional case. To illustrate this, consider a ten-dimensional dataset with 100 components per dimension. The total number of components needed would be $Q = 100^{10} = 10^{20}$, which is computationally infeasible. This rapid increase in the number of components makes the original GSM kernel impractical for high-dimensional data. In Section \ref{sec:GSM-MISO}, we mitigate this exponential growth of $Q$ by proposing a new formulation for the GSM kernel.

\subsection{Hyper-parameter Optimization}
\label{subsec:vSCA-hyper-opt}
In the GPR model, hyper-parameter optimization typically involves maximizing the model log marginal likelihood \cite{theodoridisMachineLearningBayesian2020}, which can be computed analytically due to the Gaussian assumption. Consequently, learning the hyper-parameters $\btheta \!=\! [\theta_1, \ldots, \theta_Q]^\top$ in the GSM kernel can be formulated into the following equivalent minimization problem,
\begin{equation}
\hat{\bm{\theta}} = \arg \min_{\bm{\theta}} \, \underbrace{\bm{y}^\top  \bm{C}^{-1}(\bm{\theta})\bm{y} + \log \det \left(\bm{C}(\bm{\theta})\right)}_{\triangleq l(\btheta)} + \, \text{constant},
\label{eq:DCP}
\end{equation}
where $\bm{C}(\bm{\theta}) \triangleq \bm{K}_{XX} + \sigma_{e}^2 \bm{I}_n$. Commonly, gradient-based methods, such as BFGS-Newton and conjugate gradient descent \cite{rasmussenGaussianProcessesMachine2006}, are employed to solve the above minimization problem. However, these methods encounter two primary challenges. First, these methods involve kernel matrix inversion, leading to a prohibitive computational complexity of $\mathcal{O}(n^3)$. Second, the non-convex nature of the objective often causes these methods to easily get stuck in sub-optimal solutions.

Indeed, upon careful examination of the LMK case (e.g., the GSM kernel), the optimization problem in Eq.~(\ref{eq:DCP}) is a well-known difference-of-convex programming (DCP) problem \cite{boydConvexOptimization2004}, where $g(\btheta) \!\triangleq\! \bm{y}^\top  \bm{C}^{-1}(\bm{\theta})\bm{y}$ and $h(\btheta) \!\triangleq\! -\log \det \left(\bm{C}(\bm{\theta})\right)$ are convex functions with respect to $\btheta$.  By exploiting this difference-of-convex structure, the successive convex approximation (SCA) algorithm can be applied to solve the optimization problem efficiently \cite{scutariParallelDistributedSuccessive2018}. The basic principle behind the SCA algorithm is to solve the problem in Eq.~(\ref{eq:DCP}) iteratively, where at each iteration a so-called \textit{surrogate function} $\tilde{l}(\btheta , \btheta^\eta ): \Theta \times \Theta \mapsto \mathbb{R}$ is minimized, i.e.,
\begin{equation}
\btheta^{\eta +1} = \arg \min_{\btheta} \, \tilde{l}(\btheta, \btheta^{\eta }).
\label{eq:convdcp}
\end{equation}
Here, $\tilde{l}(\btheta , \btheta^\eta )$ is assumed to be: 1) strongly convex on $\Theta$, and 2) differentiable with $\nabla_{\boldsymbol{\theta}} \tilde{l}(\btheta , \btheta^\eta ) \!\!=\!\! \nabla_{\boldsymbol{\theta}} l(\btheta) \big|_{\btheta = \btheta^\eta }$.
By fulfilling these conditions, the SCA algorithm is guaranteed to converge to a stationary solution \cite{scutariParallelDistributedSuccessive2018}. In our DCP problem, one way to construct such convex surrogate function is to make $h(\btheta)$ affine by performing the first-order Taylor expansion,
\begin{equation}
\label{eq:sca-surrogate}
\begin{aligned}
\tilde{l}(\btheta, \btheta^\eta ) &= g(\btheta) - h(\btheta^\eta ) - \nabla_{\boldsymbol{\theta}} h(\btheta^\eta )^\top (\btheta - \btheta^\eta ).\\
\end{aligned}
\end{equation}
Hence, the convex minimization problem formulated in Eq.~\eqref{eq:convdcp} can be solved efficiently using the commercial solver MOSEK \cite{mosek, yinLinearMultipleLowRank2020}, with a computational complexity of $\mathcal{O}(Qn^3)$ per iteration. 
It is also noteworthy that the constructed surrogate function in Eq.~(\ref{eq:sca-surrogate}) is a global upper bound of the objective function $l(\btheta)$, so the above SCA algorithm actually coincides with the majorization-minimization (MM) algorithm \cite{sunMajorizationMinimizationAlgorithmsSignal2017} applied to the original GSM kernel learning \cite{yinLinearMultipleLowRank2020}.

Despite the efficient convex optimization framework offered by SCA, it encounters two primary computational limitations that need to be addressed. First, its centralized computation does not scale well to large or high-dimensional datasets, with the associated computational complexity scaling as $\mathcal{O}(Qn^3)$. Second, SCA requires an access to the full dataset at each iteration, which is prohibitive for massive distributed datasets prevalent in real-world applications. This also raises privacy concerns, as individual data must be aggregated to a central unit for processing.

\section{GSMP Kernel for Multi-dimensional Data}
\label{sec:GSM-MISO}

The original GSM kernel, introduced in Section \ref{subsec:GSM}, is primarily designed for one-dimensional input data ($P=1$), limiting its utility for multi-dimensional data. When applied to multi-dimensional data, the GSM kernel experiences an exponential increase in both model and computational complexity due to the rapid growth in the number of candidate grid points \cite{suwandiGaussianProcessRegression2022}. This exponential growth poses computational challenges and complicates the application of GSM kernel in high-dimensional problems.  

\subsection{The Proposed GSMP Kernel}
\label{subsec:gsmp}
To address the limitations of the GSM kernel, we propose a new formulation called the grid spectral mixture product (GSMP) kernel. The GSMP kernel is formulated by taking the product of the one-dimensional sub-kernels, $k_q\left(\tau^{(p)}\right)$, across all dimensions, 
\begin{equation}
\setlength{\abovedisplayskip}{4pt}
\setlength{\belowdisplayskip}{4pt}
     k(\bm{\tau}) \!=\! \sum_{q=1}^{Q} \theta_{q} \prod_{p=1}^{P}  \underbrace{\exp \left\{-2 \pi^{2} {\tau^{(p)^2}} v_{q}^{{(p)}}\right\} \cos \left(2 \pi \tau^{(p)} \mu_{q}^{(p)}\right)}_{\text{one-dimensional sub-kernel:} \ k_q\left(\tau^{(p)}\right)},
    \label{eq:gsmp}
\end{equation}
where ${v}_q^{(p)}$ and $\mu_q^{(p)}$ are fixed for dimension $p$, $\forall p \! \in \! \{1, 2, \ldots, P\}$, as in the original one-dimensional GSM kernel, see Eq.~\eqref{eq:GSM-Kernel}; and $\{\theta_q\}_{q=1}^Q$ are the weights to be optimized. The GSMP kernel formulated in Eq. (\ref{eq:gsmp}) is essentially an LMK, preserving the favorable difference-of-convex structure for hyper-parameter optimization. 

To construct the GSMP kernel, we need to generate fixed grid points for each sub-kernel component $k_q(\tau^{(p)})$ along each dimension 
$p$, $p \in \{1, 2, \ldots, P\}$. We first fix the variance variables $v_q^{(p)}, \forall q$, $p$ to a small constant, e.g., 0.001. Then, we empirically sample $Q$ frequencies, either uniformly or randomly, from the frequency range $[0, \mu_{u}^{(p)}]$ to obtain $[\mu_1^{(p)}, \ldots, \mu_Q^{(p)}]$, where $\mu_{u}^{(p)}$ is the highest frequency set to be equal to $1/2$ over the minimum input spacing between two adjacent training data points in the $p$-th dimension. Using the sampled frequencies and fixed variance, we can evaluate the GSMP kernel in Eq. \eqref{eq:gsmp}, leaving the weights $\{\theta_q\}_{q=1}^Q$ to be tuned. In our experiments, without any further specification, we set the number of components $Q$ to scale linearly with the number of dimensions. This allows the GSMP kernel to effectively alleviate the curse of dimensionality while ensuring the stability of the optimization process.

By applying a Fourier transform on the GSMP kernel, we can obtain the corresponding spectral density of the GSMP kernel, which appears as a Gaussian mixture:
\begin{theorem}
\label{thm:GSMP_GMM}
   The spectral density of the GSMP kernel, defined in Eq.~\eqref{eq:gsmp}, is a Gaussian mixture given by,
    \begin{equation}
    \setlength{\abovedisplayskip}{4pt}
    \setlength{\belowdisplayskip}{4pt}
        \begin{aligned}
         S(\boldsymbol{\omega}) = \frac{1}{2^P} \sum_{q=1}^{Q} \theta_{q}  \prod_{p=1}^{P}    & \left[ \mathcal{N}   \left(\omega^{(p)} \vert \mu_q^{(p)}, v_q^{(p)}  \right) \right.\\
            & + \left. \mathcal{N}   \left(\omega^{(p)} \vert - \mu_q^{(p)}, v_q^{(p)} \!\right) \right].
        \end{aligned}
        \label{eq:gsmp-density_}
    \end{equation}
\end{theorem}
\begin{proof}
    The proof can be found in Appendix \ref{appendix:density-gsmp}.
\end{proof}
It is worth noting that Gaussian mixtures are dense in the set of all distributions \cite{hatzinakosGaussianMixturesTheir2001}. Hence, the Fourier dual of this set, which corresponds to the set of stationary kernels, is also dense \cite{wilsonGaussianProcessKernels2013}. 
That is to say, Theorem \ref{thm:GSMP_GMM} implies that, given large enough components $Q$, the GSMP kernel has the ability to approximate the underlying stationary kernel well \cite{yinLinearMultipleLowRank2020}.

\subsection{Properties of GSMP Kernel}
\label{subsec:gsmp-properties}
The GSMP kernel exhibits several advantageous properties that enhances its utility over existing formulations like the GSM kernel. One good property of the GSMP kernel is that it requires significantly fewer grid points to achieve a good approximation capability compared to the multi-dimensional GSM kernel proposed in \cite{suwandiGaussianProcessRegression2022}. We illustrate this property with a simple example.
\begin{example} \label{example:1}
In the 2-dimensional space, according to Eq.~\eqref{eq:gsmp-density_},  the spectral density of the GSM kernel is given by \cite[Eq. (11)]{suwandiGaussianProcessRegression2022},
\begin{align}
    \label{eq:gsm-density}
    &S_{\mathrm{GSM}}(\bm{\omega}) = \sum_{q=1}^{Q} \frac{\theta_q}{2} \left[  
     \mathcal{N}\left(\omega^{(1)}\vert \mu_q^{(1)}, v_q^{(1)}\right) \mathcal{N}\left(\omega^{(2)}\vert \mu_q^{(2)}, v_q^{(2)}\right) \right. \nonumber \\
    & \left. \hspace{4ex} + \ \mathcal{N}\left(\omega^{(1)}\vert - \mu_q^{(1)}, v_q^{(1)}\right) \mathcal{N}\left(\omega^{(2)}\vert - \mu_q^{(2)}, v_q^{(2)}\right) \right].
\end{align}
while the spectral density of the GSMP kernel is,
\begin{equation}
    \label{eq:gsmp-density}
    \begin{aligned}
     &S_{\mathrm{GSMP}}(\boldsymbol{\omega}) = 
     \sum_{q=1}^{Q} \frac{\theta_q}{4}  \left[  
     \mathcal{N}\left(\omega^{(1)}\vert \mu_q^{(1)}, v_q^{(1)}\right) \mathcal{N}\left(\omega^{(2)}\vert \mu_q^{(2)}, v_q^{(2)}\right)  \right. \\
     & \left. \hspace{6ex} + \ \mathcal{N}\left(\omega^{(1)}\vert - \mu_q^{(1)}, v_q^{(1)}\right) \mathcal{N}\left(\omega^{(2)}\vert - \mu_q^{(2)}, v_q^{(2)}\right)  \right.  \\
     & \left. \hspace{6ex} + \ \mathcal{N}\left(\omega^{(1)}\vert \mu_q^{(1)}, v_q^{(1)}\right) \mathcal{N}\left(\omega^{(2)}\vert -\mu_q^{(2)}, v_q^{(2)}\right)  \right.  \\
     & \left. \hspace{6ex} + \ \mathcal{N}\left(\omega^{(1)}\vert -\mu_q^{(1)}, v_q^{(1)}\right) \mathcal{N}\left(\omega^{(2)}\vert \mu_q^{(2)}, v_q^{(2)}\right) \right],
\end{aligned}
\end{equation}
The GSMP density contains the $2Q$ densities of the GSM plus additional $2Q$ Gaussian densities, for approximating the underlying density, as demonstrated in Eqs.~\eqref{eq:gsm-density} and \eqref{eq:gsmp-density}.  
\end{example}

As illustrated in Example \ref{example:1}, the spectral density of the GSMP kernel exhibits a greater number of Gaussian densities for approximating the spectral density of the underlying kernel, compared to that of the GSM kernel. Extending this to a $P$-dimensional space, the number of additonal Gaussian densities is $Q \times (2^P - 2)$ for the same number of grid points, $Q$, and dimensionality, $P$. Consequently, the GSM kernel requires an additional $Q\times(2^{P-1} - 1)$ grid points or sub-kernels to match GSMP kernel's capacity. This structural advantage of the GSMP kernel, facilitating fewer preselected grid points or sub-kernels, not only alleviates the curse of dimensionality but also reduces both the model and computational complexity. Further results and discussions are deferred to Section \ref{subsec:gsmp-vs-gsm}.

Another property of the GSMP kernel is that, the solution to the corresponding hyper-parameter optimization problem formulated in Eq. (\ref{eq:DCP}) is sparse, as succinctly summarized in the following theorem.
\begin{theorem}
\label{thm:sparsity}
Every local minimum of the hyper-parameter optimization problem formulated in Eq. (\ref{eq:DCP}) for the proposed GSMP kernel, is achieved at a sparse solution, regardless of whether noise is present or not.
\begin{proof}
    The proof can be found in Appendix \ref{appendix:sparsity}.
\end{proof}
\end{theorem}
The solution's sparsity is a notable advantage for the proposed GSMP kernel. First, it ensures that only the grid points or sub-kernels deemed significant by the data are pinpointed, enhancing the kernel's interpretability. Second, by refraining from utilizing all grid points for data fitting, the issue of over-parameterization can be effectively alleviated. In Section \ref{sec:SLIM-KL}, we demonstrate how to exploit this inherent sparsity to design an efficient optimization framework for learning the hyper-parameters.

\section{Sparsity-Aware Distributed Learning Framework}
\label{sec:SLIM-KL}
Previously, the SCA algorithm introduced in Section \ref{subsec:vSCA-hyper-opt} suffers from scalability issues as solving the optimization problem in Eq. (\ref{eq:convdcp}) becomes computationally prohibitive for big data, i.e., with large $n$. Moreover, large amounts of labeled training data are usually aggregated from a large number of local agents or mobile devices, which may cause severe data privacy issues \cite{yinFedLocFederatedLearning2020}. Practical constraints faced by agents, such as limited bandwidth, battery power, and computing resources \cite{kashyapQuantizedConsensus2007}, further exacerbate these issues. For instance, the extensive information exchange between diverse local agents often exceed the available bandwidth for data transmission, as discussed in \cite{amiriFederatedLearningQuantized2020}. 

To address these challenges, we propose an efficient, sparsity-aware learning framework, termed Sparse LInear Multiple Kernel Learning (SLIM-KL), depicted in Fig.~\ref{fig:SLIM-KL}, for learning the proposed kernel. Unlike the doubly distributed SCA framework proposed in \cite{suwandiGaussianProcessRegression2022}, SLIM-KL integrates a quantized alternating direction method of multipliers (ADMM) scheme to enable collaborative learning with reduced communication costs and practical consideration of transmission constraints. This quantization scheme is a key differentiator, ensuring efficient transmission of model hyper-parameters between distributed agents. Additionally, SLIM-KL employs a distributed successive convex approximation (DSCA) algorithm for scalable training of a large number of hyper-parameters while leveraging the sparsity-promoting property of the proposed GSMP kernel. Note that while this paper focuses on the GSMP kernel, the proposed learning framework can be applied to any LMKs. The key notations used in our SLIM-KL framework are summarized in Table \ref{tab:notations}. In the following two subsections, we detail the two key components of our SLIM-KL framework.

\begin{figure*}[!t]
    \centering
    \includegraphics[width=.8\linewidth]{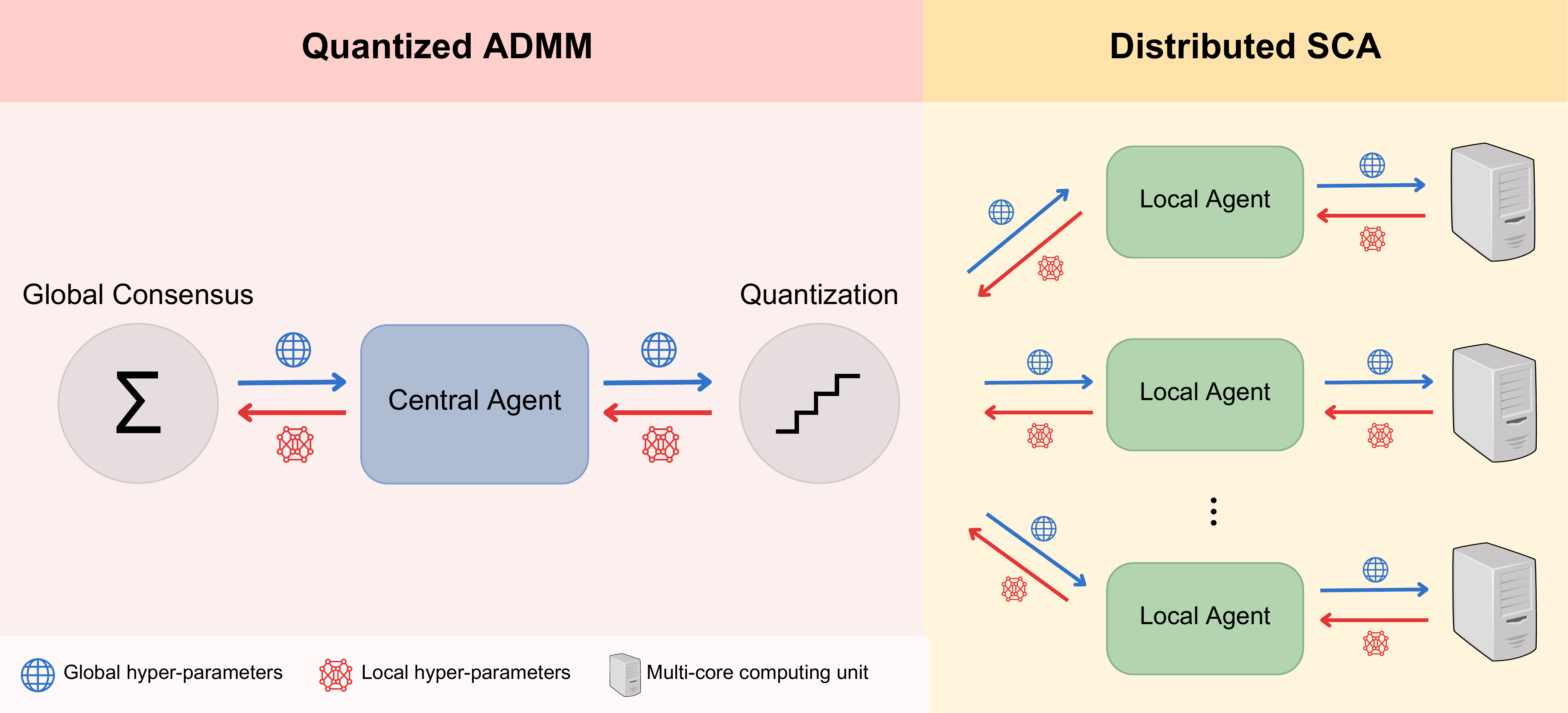}
    \caption{The proposed \textbf{Sparse Linear Multiple Kernel Learning (SLIM-KL)} framework, featuring a quantized ADMM scheme for collaborative hyper-parameter learning across multiple agents, and a distributed SCA algorithm for local optimization using multi-core computing units.}
    \label{fig:SLIM-KL}
\end{figure*}

\begin{table}[!t]
    \centering
    \caption{Summary of key notations.}
    \label{tab:notations}
    \begin{tabular}{@{}cl@{}}
    \toprule
    \textbf{Notation} & \textbf{Description} \\
    \midrule
    $n$ & Training size \\
    $P$ & Input dimension of the data \\ 
    $Q$ & Number of hyper-parameters (grid points) \\
    $N$ & Number of local agents \\
    $s$ & Number of computing units \\
    $\boldsymbol{\theta}$ & Global hyper-parameters (weights of GSMP kernel) \\
    $\boldsymbol{\zeta}_j$ & Local hyper-parameters of the $j$-th agent \\
    $\mathcal{D}_j$ & Local dataset owned by the $j$-th agent \\
    $l$ & Objective function (negative log marginal likelihood of the GP) \\
    $\tilde{l}$ & Surrogate function for the objective \\
    $\boldsymbol{\lambda}_j$ & Dual variable in ADMM \\
    $\rho_j$ & Penalty parameter in ADMM \\
    $\mathcal{Q}$ & Stochastic quantization function \\
    $\Delta$ & Quantization resolution \\
    \bottomrule
    \end{tabular}
\end{table}

\subsection{Quantized ADMM}
\label{subsec:quantized-ADMM}
ADMM has emerged as a promising algorithm for developing a principled distributed learning framework \cite{boydDistributedOptimizationStatistical2011}, enabling multiple agents solve large-scale distributed GPR problems \cite{xuWirelessTrafficPrediction2019}. In this work, we consider a variant of ADMM that incorporates quantization \cite{zhuQuantizedConsensusADMM2016a}, as illustrated in Fig.~\ref{fig:SLIM-KL} (Quantized ADMM). Quantization is a widely used technique in wireless systems to compress information by reducing transmission bits \cite{zhaoParticle2015, jinDithering2015}. Mathematically, consider $N \!\in\! \mathbb{N}$ multi-core agents cooperatively learning a global GPR model with the proposed GSMP kernel, the learning problem can be formulated as, 
    \begin{equation}
    \label{eq:disGP}
        \begin{aligned}
    	& \min  \,  \sum_{j=1}^{N} l(\bzeta_j; \mathcal{D}_j) 	\\
    	& ~ \operatorname{s.t.} \ \bzeta_j - \btheta = \bm{0}, \ j \in \{1,\ldots, N\},\\
    	& \  \qquad \btheta \in \Theta, \ \bzeta_j \in \Theta, \ j \in \{1,\ldots, N\},
        \end{aligned}
    \end{equation}
where $\mathcal{D}_j$ denotes local dataset owned by the $j$-th agent with $|\mathcal{D}_j| = n/N$, $\bzeta_j$ denotes the local hyper-parameters, and $\btheta$ denotes the global hyper-parameters (the weights of GSMP kernel). In this setup, we assume that each agent has the ability to store and perform computations with infinite precision. However, when it comes to transmission, all agents are constrained to sending quantized data that can be received without any error. 

Let $\boldsymbol{\theta}^t$ denote the hyper-parameters at iteration $t$. Suppose for each element $i \in \{1, \ldots, Q\}$ of $\boldsymbol{\theta}^t$, it holds that $\theta_i^t \in [\theta_{\mathrm{min}}^t, \theta_{\mathrm{max}}^t]$, where $\theta_{\mathrm{min}}^t$ and $\theta_{\mathrm{max}}^t$ denote the smallest and largest entries of $\boldsymbol{\theta}^t$, respectively.  We aim to obtain a quantized value of $\theta_i^t$ with length $l_t$ bits. We define $L_t = 2^{l_t}$ quantization points given by the set $\Lambda = \{\tau_1, \ldots, \tau_L\}$, where $\tau_1 = \theta_{\mathrm{min}}^t$ and $\tau_L = \theta_{\mathrm{max}}^t$. The points in this set are uniformly spaced such that $\Delta = \tau_{j + 1} - \tau_j$ for $j \in \{1, \ldots, L_t -1\}$ \cite{doan2020fast}. It follows that, $\Delta = (\theta_{\mathrm{max}}^t - \theta_{\mathrm{min}}^t)/({2^{l_t} - 1})$. Here, $\Delta > 0$ denotes the quantization resolution, where smaller $\Delta$ allows more precise transmissions but requires more bits and vice versa. Then, for $\theta_i^t \in \ [\tau_j, \tau_{j + 1})$, we can quantize $\theta_i^t$ using the following stochastic quantization function \cite{xiaoDecentralizedEstimationInhomogeneous2005},
\begin{align}\label{def:stochastic-quantization}
    \mathcal{Q}(\theta_i^t) =
\begin{cases}
\tau_j, & \text{with probability} \ 1 - (\theta_i^t - \tau_j)/\Delta, \\
\tau_{j + 1}, & \text{with probability} \ (\theta_i^t - \tau_j)/\Delta.
\end{cases}
\end{align}

In Section \ref{sec:analysis}, we will establish that quantized ADMM is guaranteed to converge to a stationary point, thanks to the properties of the stochastic quantization used.  

To solve the optimization problem in Eq.~(\ref{eq:disGP}), we first formulate the augmented Lagrangian as 
\begin{equation}
\begin{aligned}
& \quad \mathcal{L}(\bzeta_1, \ldots, \bzeta_N, \blambda_1, \ldots, \blambda_N, \btheta)\\
&  \triangleq \sum_{j=1}^N \left( l(\bzeta_j; \mathcal{D}_j) + \blambda_j^\top (\bzeta_j - \btheta)  + \frac{\rho_j}{2} \|\bzeta_j - \btheta \|_2^2\right),
\end{aligned}
\end{equation}
where $\blambda_{j}$ is the dual variable and $\rho_j$ is a preselected penalty parameter.  

By imposing a quantization operation on $\btheta$ and $\bzeta_j$ immediately after the $\btheta$-update and $\bzeta_j$-update, the sequential update in quantized ADMM for the $(t + 1)$-th iteration can be decomposed as
    \begin{subequations}
    \begin{align}
		 \btheta^{t+1} &=\frac{1}{N} \sum_{j=1}^{N}\left(\boldsymbol{\zeta}_{j,[\mathcal{Q}]}^{t}+\frac{1}{\rho_j} \blambda_{j}^{t}\right), \label{eq:qadmm-consensus}\\
         \btheta^{t+1}_{[\mathcal{Q}]} &= \mathcal{Q}(\btheta^{t+1}) \label{eq:qadmm-quantization1}, \\
		\boldsymbol{\zeta}_{j}^{t+1} &\! =\! \arg \min_{\boldsymbol{\zeta}_{j}} \, l\left(\boldsymbol{\zeta}_{j}; \mathcal{D}_j\right) + \left(\boldsymbol{\lambda}_{j}^{t}\right)^{\top}\left(\boldsymbol{\zeta}_{j}- \btheta^{t+1}_{[\mathcal{Q}]}\right) \nonumber \\  & \qquad \qquad + \frac{\rho_j}{2}\left\|\boldsymbol{\zeta}_{j}-\btheta^{t+1}_{[\mathcal{Q}]}\right\|_{2}^{2}, \ j \in \{1,\ldots, N \},\label{eq:qadmm-update} \\
        \bzeta^{t+1}_{j, [\mathcal{Q}]} &= \mathcal{Q}(\bzeta_j^{t+1}) \label{eq:qadmm-quantization2}, \\
		\boldsymbol{\lambda}_{j}^{t+1} &=\boldsymbol{\lambda}_{j}^{t}+\rho_j\left(\boldsymbol{\zeta}_{j, [\mathcal{Q}]}^{t+1}-\btheta^{t+1}_{[\mathcal{Q}]}\right).\label{eq:qadmm-dual}
        \end{align}
    \end{subequations}
Eq.~(\ref{eq:qadmm-consensus}) is the global consensus step, averaging the local agents' updates to form the new global hyper-parameter $\boldsymbol{\theta}^{t+1}$. Eq.~(\ref{eq:qadmm-update}) is the local update step, where each agent $j$ optimizes its local hyper-parameters $\boldsymbol{\zeta}_j$ to align with the global hyper-parameter using its local data $\mathcal{D}_j$. Eq.~(\ref{eq:qadmm-dual}) updates the dual variables $\boldsymbol{\lambda}_j$ to maintain consistency between the global and local hyper-parameters. Eq.~(\ref{eq:qadmm-quantization1}) and (\ref{eq:qadmm-quantization2}) denote the global and local quantization steps, respectively. This quantized ADMM scheme allows multiple agents to collaboratively learn the hyper-parameters while preserving data privacy and reducing communication cost. Another useful property of the quantized ADMM is that if one local agent gets stuck at a poor local minimum of the GSMP kernel hyper-parameter, the consensus step may help the agent resume from a more reasonable point in the next iteration.

\subsection{DSCA Algorithm}
\label{subsec:DSCA}
 It is worth noting that the local minimization problem formulated in Eq.~\eqref{eq:qadmm-update} is not only non-convex but also non-separable, rendering distributed computation infeasible. Hence, solving it with respect to the entire $\bzeta_j$ can be computationally demanding for large $Q$. To alleviate this computational burden, we propose the distributed SCA (DSCA) algorithm to optimize the hyper-parameters in a distributed manner, as depicted in Fig. \ref{fig:SLIM-KL} (Distributed SCA). The key idea is to split the hyper-parameters $\bzeta_j$ into $s$ blocks and optimize them in parallel across $s$ computing units, $s \!\in\! \mathbb{N}$. Concretely, by exploiting the Cartesian product structure of the feasible set $\Theta$ in the GSMP kernel, i.e., $\Theta = \Theta_1 \times \Theta_2 \times \ldots \times \Theta_s$ with $\Theta_i \subseteq \mathbb{R}^{Q/s}$, we first partition the optimization variable into $s$ blocks, $\bzeta_j = [\bzeta_{j,1,} \bzeta_{j,2}, \ldots, \bzeta_{j,s}]^\top$. Then, we can construct the associated surrogate function for $l\left(\boldsymbol{\zeta}_{j}; \mathcal{D}_j\right)$ in Eq.~\eqref{eq:qadmm-update}, that is additively separable in the blocks:
	\begin{align}
	\tilde{l}(\bzeta_j, \bzeta_j^\eta) = \sum_{i = 1}^s \tilde{l}_i (\bzeta_{j,i} ,  \bzeta_j^\eta).
	\label{eq:DSCA}
	\end{align}
	\begin{assumption}
	\label{assumption:surrogate-dsca}
 Each surrogate function $\tilde{l}_i (\bzeta_{j,i} ,  \bzeta_j^\eta): \Theta_i \times \Theta \mapsto \mathbb{R}, i \in \{1, \ldots, s\}$, satisfies the following conditions:
	\begin{enumerate}
		\item $\tilde{l}_i(\bzeta_{j,i} , \bzeta_j^\eta)$ is strongly convex on $\Theta_i$;
		\item $\tilde{l}_i(\bzeta_{j,i} ,  \bzeta_j^\eta)$ is differentiable with 
		\begin{equation}
                \nabla_{\bzeta_{j,i}} \tilde{l}_i(\bzeta_{j,i}, \bzeta_j^\eta) = \nabla_{\bzeta_{j,i}} l(\bzeta_j)\big|_{\bzeta_j = \bzeta_j^\eta}.
		\end{equation}
	\end{enumerate}
	\end{assumption}
To fulfill the two conditions given in Assumption \ref{assumption:surrogate-dsca}, we employ a similar strategy as the vanilla SCA, and construct $\tilde{l}_i (\bzeta_{j,i} ,  \bzeta_j^\eta)$ in Eq.~(\ref{eq:DSCA}) as,
	\begin{equation}
    \label{eq:dsca-surrogate}
	\begin{aligned}
	\tilde{l}_i(\bzeta_{j,i}, \bzeta_j^\eta) \!=\! g(\bzeta_{j,i}, \bzeta_{j, -i}^\eta) - h(\bzeta_j^\eta) - \nabla_{\bzeta_{j,i}} h(\bzeta_j^\eta)^\top (\bzeta_{j,i} - \bzeta_{j,i}^\eta),
	\end{aligned}
	\end{equation}
	where 
	\begin{equation}
	\nabla_{\bzeta_{j,i}} h(\bzeta_j^\eta) = \left.\begin{bmatrix}
	-\operatorname{Tr}\left( \boldsymbol{C}^{-1} (\bzeta_j) \frac{ \partial \boldsymbol{C} (\bzeta_j)}{\partial \zeta_{j, 1+(i-1)Q/s}} \right)\\ 
	-\operatorname{Tr}\left( \boldsymbol{C}^{-1} (\bzeta_j) \frac{ \partial \boldsymbol{C} (\bzeta_j)}{\partial \zeta_{j,2+(i-1)Q/s}} \right) \\
	\vdots\\
	-\operatorname{Tr}\left( \boldsymbol{C}^{-1} (\bzeta_j) \frac{ \partial \boldsymbol{C} (\bzeta_j)}{\partial \zeta_{j, iQ/s} } \right)
	\end{bmatrix}\right|_{\bzeta_j = \bzeta_j^\eta},
	\nonumber
	\end{equation}
	and 
	$\bzeta_{j, -i}^\eta = \left[\bzeta_{j, 1}^\eta, \ldots, \bzeta_{j, i - 1}^\eta, \bzeta_{j, i + 1}^\eta, \ldots, \bzeta_{j, s}^\eta\right]^\top$. 
 The DSCA algorithm solves the following convex optimization problem in parallel for all $i \in \{1, \ldots, s\}$,
    \begin{equation}
    \label{eq:dsca_problem}
        \begin{aligned}
        {\bzeta}_{j, i}^{\eta+1} = \arg \min_{\bzeta_{j, i}} \, \ell(\bzeta_{j, i},  \bzeta_j^\eta),
        \end{aligned} 
    \end{equation}
    where $\ell(\bzeta_{j, i}, \bzeta_j^\eta) \! \triangleq \! \tilde{l}_i(\bzeta_{j, i}, \bzeta_j^\eta) \!+\! \left(\boldsymbol{\lambda}_{j, i}^{t}\right)^{\top} \!\! \left(\bzeta_{j, i} \!-\! \btheta_i^{t+1}\right) \!+\! \frac{\rho_j}{2}\left\|\bzeta_{j, i}-\btheta_i^{t+1}\right\|_{2}^{2}$. However, it has been observed that directly using CVX \cite{cvx}, a package designed for constructing and solving convex programs, to solve Eq. (\ref{eq:dsca_problem}) is computationally demanding. Since the term $\tilde{l}_i(\bzeta_{j, i}, \bzeta_j^\eta)$ in Eq. (\ref{eq:dsca_problem}) contains a matrix fractional component, we can reformulate the problem into a conic form, by introducing auxiliary variables $\bm{z}, \bm{w}, v$:
\begin{align} \label{eq:socp_d2sca_conic}
        &\min_{\bzeta_{j, i}, \bm{z}, \bm{w}, v} \, (\bm{1}^\top \bm{z}) + \left(\blambda_{j, i}^t - \nabla_{\bzeta_{j, i}} h(\bzeta_j^\eta)\right)^\top \bzeta_{j, i} +\frac{\rho_j}{2} v \nonumber \\
        &\text{s.t.} \, \enspace \left\| 
    \begin{bmatrix}
    2\bm{w}_0 ,\\
    z_0 - 1
    \end{bmatrix}
    \right\|_2
    \leq z_0 + 1 \nonumber ,\\
        & \hspace{4.15ex} \left\| 
    \begin{bmatrix}
    2\bm{w}_k ,\\
    z_k - \bzeta_{j, k  + (i - 1)Q/s}
    \end{bmatrix}
    \right\|_2
    \leq z_k + \bzeta_{j, k  + (i - 1)Q/s}, \nonumber \\
    & \hspace{30.5ex} k \in \{ 1,2, \ldots, Q/s\} \nonumber ,\\
        & \hspace{4.5ex} \bm{y} = \sum_{k = 1}^{Q/s} \bm{L}_{k + (i - 1)Q/s} \bm{w}_k \nonumber \\ 
        & \hspace{5ex} + \left( \sum_{n \neq i}\sum_{k = 1}^{Q/s} \bzeta_{j, k  + (n - 1)Q/s}^\eta \bm{K}_{k  + (n - 1)Q/s} + \sigma_e^2 \bm{I}_n \right)^{\frac{1}{2}} \!  \bm{w}_0 \nonumber ,\\
        & \hspace{4.75ex}  2v \frac{1}{2} \geq \sum_{k = 1}^{Q/s} \left(\bzeta_{j, k + (i - 1)Q/s}-\btheta_{k + (i - 1)Q/s}^{t+1}\right)^2 \nonumber ,\\
        & \hspace{4.75ex} \bzeta_{j, i} \geq \bm{0}, \enspace \bm{z} \geq \bm{0}, \enspace v \geq 0.
\end{align}
We provide the detailed steps for this reformulation in Appendix \ref{appendix:reformulation}. The problem presented in Eq. (\ref{eq:socp_d2sca_conic}) is now in conic form and can be solved more efficiently using MOSEK \cite{mosek, yinLinearMultipleLowRank2020}. After the local agents have optimized their local hyper-parameters using DSCA in the $t$-th iteration, the ADMM consensus step formulated in Eq. (\ref{eq:qadmm-consensus}) is conducted to reach a global consensus. 
At this point, we have introduced the SLIM-KL framework. The pseudocode for implementing  SLIM-KL is summarized in Algorithm \ref{alg:SLIM-KL}.

\begin{remark}
    \label{remark:complexity}
    By utilizing $N$ local agents and $s$ computing units, the overall computational complexity of SLIM-KL at each iteration is reduced from $\mathcal{O}(Qn^3)$ to $\mathcal{O}(\frac{Qn^3}{s N^3})$. That is, as the number of agents $N$ and computing units $s$ increase, the computational burden of optimizing the hyper-pameters can be significantly reduced, especially when dealing with large number of hyper-parameters $Q$ and training samples $n$.
\end{remark}

\begin{remark}
\label{remark:slim-kl}
The design of SLIM-KL matches perfectly with the sparse nature of the proposed GSMP kernel, see Theorem \ref{thm:sparsity}. Many weights/hyper-parameters of the sub-kernels in the GSMP kernel are inherently zero or near-zero valued, so quantization can enhance this sparsity by rounding even more of these weights to zero. This enhanced sparsity could aid convergence and guide the algorithm toward a favorable local minimum, as noted in \cite{zhuQuantizedConsensusADMM2016a}. 
\end{remark}

\begin{algorithm}[!t]
    \caption{SLIM-KL for LMK learning}\label{alg:SLIM-KL}
    \textbf{Initialization}: $t=0$, $\{ \rho_j, \bm{\zeta}_j^0, \bm{\lambda}_j^0\}_{j=1}^N$,
    $N$ local datasets $\{\mathcal{D}_j\}_{j=1}^N$ in $N$ local computing agents, quantization function $\mathcal{Q}(\cdot)$;\\
    \While{Iteration stopping criteria are not satisfied}{
        Obtain the global hyper-parameters $\bm{\theta}^{t+1}$ by Eq. (\ref{eq:qadmm-consensus});

        Apply quantization on the global hyper-parameters using Eq. (\ref{eq:qadmm-quantization1}) and obtain $\btheta^{t+1}_{[\mathcal{Q}]}$; 
        
        \For{$j \in [1, N]$}{
            Obtain the local hyper-parameters $\bm{\zeta}_j^{t+1}$ by utilizing DSCA to solve Eq. (\ref{eq:socp_d2sca_conic}) in parallel using $s$ units;

            Apply quantization on the local hyper-parameters using Eq. (\ref{eq:qadmm-quantization2}) and obtain $\bzeta^{t+1}_{j, [\mathcal{Q}]}$; 
        }
        
        Obtain the dual variables by Eq. (\ref{eq:qadmm-dual});		
    }
    \textbf{return} $\boldsymbol{\theta}^t$
\end{algorithm}

\section{Analysis}
\label{sec:analysis}
In this section, we present the theoretical analysis of the SLIM-KL framework proposed in Section \ref{sec:SLIM-KL}. First, we show that the DSCA algorithm solving the local minimization problem within quantized ADMM is guaranteed to converge, which is succinctly summarized as follows.
\begin{theorem}
\label{theorem:dsca_convergence}
     By employing DSCA, the objective of the local minimization problem, see Eq.~\eqref{eq:qadmm-update}, is non-increasing, and the sequence of feasible points $\{\bzeta_j^\eta\}_{\eta \in \mathbb{N}}$ satisfies,
    \begin{equation}
    \label{eq:dsca_cvg}
        \lim_{\eta \to \infty} \| \bzeta_j^{\eta + 1} - \bzeta_j^\eta \| = 0.
    \end{equation}
\end{theorem}
\begin{proof}
    The proof can be found in Appendix \ref{appendix:proof_theorem_dsca_coverg}.
\end{proof}
Theorem \ref{theorem:dsca_convergence} suggests that within the iteration of quantized ADMM, the local minimization problem is solved exactly \cite{hongConvergenceAnalysisAlternating2016}. Building upon the result from Theorem \ref{theorem:dsca_convergence}, we can show that the quantized ADMM-based SLIM-KL framework is guaranteed to reach a stationary point in a mean sense, according to the following theorem.
\begin{theorem}
\label{theorem:slim-kl-convergence}
    The SLIM-KL framework ensures that, under the stochastic quantization scheme, the sequence $\{\bzeta_j^t\}_{j = 1}^N$ converges to $\btheta^t$ in the mean sense as $t \to \infty$, with the variance of $\bzeta_j^t$ stabilizing at a finite value which depends on the quantization resolution $\Delta$. 
\end{theorem}
\begin{proof}

We first refer to the following lemma,
\begin{lemma}
\label{lemma:stochastic-quantization}
    With stochastic quantization, each $\theta \in \mathbb{R}$ can be unbiasedly estimated as
    \begin{align}
        \mathbb{E}[\mathcal{Q}(\theta)] = \theta,
    \end{align}
    and the associated quantization error is bounded by
    \begin{align}
        \mathbb{E}\left[(\theta - \mathcal{Q}(\theta))^2\right] \leq \frac{\Delta^2}{4}.
    \end{align}
\end{lemma}
\begin{proof}
    The proof can be found in Appendix \ref{appendix:lemma-proof}.
\end{proof}
It follows that, by taking expectation on both sides of Eq. (\ref{eq:qadmm-consensus}) and Eq. (\ref{eq:qadmm-dual}) we can get,
\begin{subequations}
    \begin{align}
     \mathbb{E}\left[\btheta^{t+1}\right] &=\frac{1}{N} \sum_{j=1}^{N}\left(\mathbb{E}\left[\boldsymbol{\zeta}_{j,[\mathcal{Q}]}^{t}\right]+\frac{1}{\rho_j} \mathbb{E}\left[\blambda_{j}^{t}\right]\right), \label{eq:qadmm-consensus-exp}\\
    \mathbb{E}\left[\boldsymbol{\lambda}_{j}^{t+1} \right] &= \mathbb{E}\left[\boldsymbol{\lambda}_{j}^{t}\right] +\rho_j\left(\mathbb{E}\left[\boldsymbol{\zeta}_{j, [\mathcal{Q}]}^{t+1}\right] -\mathbb{E}\left[\btheta^{t+1}_{[\mathcal{Q}]}\right]\right).\label{eq:qadmm-dual-exp}
\end{align}
\end{subequations}
Note that Lemma \ref{lemma:stochastic-quantization} implies that $\mathbb{E}[\boldsymbol{\zeta}_{j, [\mathcal{Q}]}^{t+1}] = \bzeta_j^{t + 1}$ and $\mathbb{E}[\btheta^{t+1}_{[\mathcal{Q}]}] = \btheta^{t + 1}$. Since $\mathbb{E}[\blambda_j^0] = \blambda_j^0$, it is easy to see that Eq. (\ref{eq:qadmm-consensus-exp}) and Eq. (\ref{eq:qadmm-dual-exp}) takes exactly the same iterations in the mean sense as the consensus ADMM, i.e., ADMM without quantization scheme. Therefore, the convergence in mean of \(\{\bzeta_j^t\}_{j = 1}^N\) to \(\btheta^t\) is guaranteed due to the convergence of consensus ADMM established in \cite{boydDistributedOptimizationStatistical2011, dengGlobalLinearConvergence2016, zhuQuantizedConsensusADMM2016a}. Moreover, let $\boldsymbol{\epsilon}^t = \boldsymbol{\zeta}^t_j - \mathcal{Q}(\boldsymbol{\zeta}^t_j)$ denote the quantization error of $\boldsymbol{\zeta}^t_j$. By Lemma \ref{lemma:stochastic-quantization}, we can show that it satisfies: $
\mathbb{E}\left[\boldsymbol{\epsilon}^t\right] = 0 \quad \text{and} \quad \mathrm{Var}\left(\boldsymbol{\epsilon}^t\right) \leq \Delta^2/4.$
That is, the quantization error $\boldsymbol{\epsilon}^t$ introduced at each iteration has an expected value of zero and a variance that is bounded by $\Delta^2/4$. As the quantization resolution $\Delta$ decreases, the variance of the quantization error also decreases, enhancing the convergence rate of the algorithm. Conversely, larger $\Delta$ increases the variance bound, potentially slowing convergence. Thus, as $t \to \infty$, the variance of $\bzeta_j^t$ stabilizes to a finite value that depends on the quantization resolution $\Delta$ \cite{zhuQuantizedConsensusADMM2016a}. This completes the proof for Theorem \ref{theorem:slim-kl-convergence}.
\end{proof}

\section{Experiments}
\label{sec:experiments}

In this section, we present experimental validations for the kernel formulation and learning framework proposed in Sections \ref{sec:GSM-MISO} and \ref{sec:SLIM-KL}. The efficiency and approximation capability of the proposed GSMP kernel are demonstrated in Section \ref{subsec:gsmp-vs-gsm}, followed by a comparative analysis between our method and some competing approaches in Section \ref{subsec:prediction}. Sections \ref{subsec:scalability} and \ref{subsec:quantization} highlight the scalability and communication efficiency of our learning framework. Lastly, Section \ref{subsec:lstf} demonstrates the performance of our method on real world long time series forecasting tasks. Our implementation code is accessible at \href{https://github.com/richardcsuwandi/SLIM-KL}{https://github.com/richardcsuwandi/SLIM-KL}.

\subsection{Approximation Capability of GSMP Kernel}
\label{subsec:gsmp-vs-gsm}
This subsection demonstrates the approximating capabilty of the GSMP kernel discussed in Section \ref{subsec:gsmp-properties}. Specifically, we empirically validate the capability of the GSMP kernel to maintain a good approximation while reducing the number of hyper-parameters, as highlighted in Example \ref{example:1}. We generate synthetic datasets with known underlying spectral densities and compare the performance of the GSMP kernel against the original GSM kernel proposed in \cite{suwandiGaussianProcessRegression2022}. The generated datasets consist of 500 samples from a GP with the SM kernel. The GP inputs are uniformly distributed within the domain $[0, 10]^2$, with 480 samples used for training and 20 samples reserved for testing. We compare the GP performance with the GSM and GSMP kernels under different configurations as outlined below.

\begin{itemize}
    \item \textbf{Configuration 1}: The SM kernel has an underlying four-mode Gaussian mixture spectral density, $S(\bm{\omega}) = \sum_{i=1}^4 \alpha_i \mathcal{N}(\bm{\mu}_i, \bm{V}_i)$, where $\alpha_1 = \ldots = \alpha_4 = 5$, $\bm{\mu}_1 = [-2,-2]^\top$, $\bm{\mu}_2 = [-2,2]^\top$, $\bm{\mu}_3 = [2,-2]^\top, \bm{\mu}_4 = [2,2]^\top$, and $\bm{V}_1 = \ldots = \bm{V}_4 = 0.1 \bm{I}$.
    \item \textbf{Configuration 2}: The SM kernel has an underlying four-mode Gaussian mixture spectral density, $S(\bm{\omega}) = \sum_{i=1}^4 \alpha_i \mathcal{N}(\bm{\mu}_i, \bm{V}_i)$, where $\alpha_1 = \ldots = \alpha_4 = 5$, $\bm{\mu}_1 = [-2,-2]^\top$, $\bm{\mu}_2 = [-3,3]^\top$, $\bm{\mu}_3 = [3,-3]^\top, \bm{\mu}_4 = [2,2]^\top$, and $\bm{V}_1 = \ldots = \bm{V}_4 = 0.1 \bm{I}$.
    \item \textbf{Configuration 3}: The SM kernel has an underlying two-mode Gaussian mixture spectral density, $S(\bm{\omega}) = \sum_{i=1}^2 \alpha_i \mathcal{N}(\bm{\mu}_i, \bm{V}_i)$, where $\alpha_1 = \alpha_2 = 5$, $\bm{\mu}_1 = [-2,-2]^\top$, $\bm{\mu}_2 = [2,2]^\top$, and $\bm{V}_1 = \bm{V}_2 = 0.1 \bm{I}$.
    \item \textbf{Configuration 4}: The SM kernel has an underlying four-mode Gaussian mixture spectral density, $S(\bm{\omega}) = \sum_{i=1}^4 \alpha_i \mathcal{N}(\bm{\mu}_i, \bm{V}_i)$, where $\alpha_1 = \alpha_4 = 15$, $\alpha_2 = \alpha_3 = 5$  $\bm{\mu}_1 = [-2,-2]^\top$, $\bm{\mu}_2 = [-3,3]^\top$, $\bm{\mu}_3 = [3,-3]^\top, \bm{\mu}_4 = [2,2]^\top$, and $\bm{V}_1 = \ldots = \bm{V}_4 = 0.1 \bm{I}$.
    \item \textbf{Configuration 5}: The SM kernel has an underlying six-mode Gaussian mixture spectral density, $S(\bm{\omega}) = \sum_{i=1}^6 \alpha_i \mathcal{N}(\bm{\mu}_i, \bm{V}_i)$, where $\alpha_1=\alpha_6 = 10, \alpha_2 = \alpha_5 = 10, \alpha_3 = \alpha_4 = 5$,  $\bm{\mu}_1 = [0.1,0.1]^\top$,  $\bm{\mu}_2 = [1, 1]^\top$, $\bm{\mu}_3 = [1.5, 1.5]^\top, \bm{\mu}_4 = [-1.5, -1.5]^\top$,  $\bm{\mu}_5 = [-1, -1]^\top$, $\bm{\mu}_6 = [-0.1, -0.1]^\top$, and $\bm{V}_1 = \ldots = \bm{V}_6 = 0.05 \bm{I}$.
\end{itemize}

\begin{figure}[!t]
\centering 
\subfloat[Configuration 1]{\includegraphics[width=.24\textwidth]{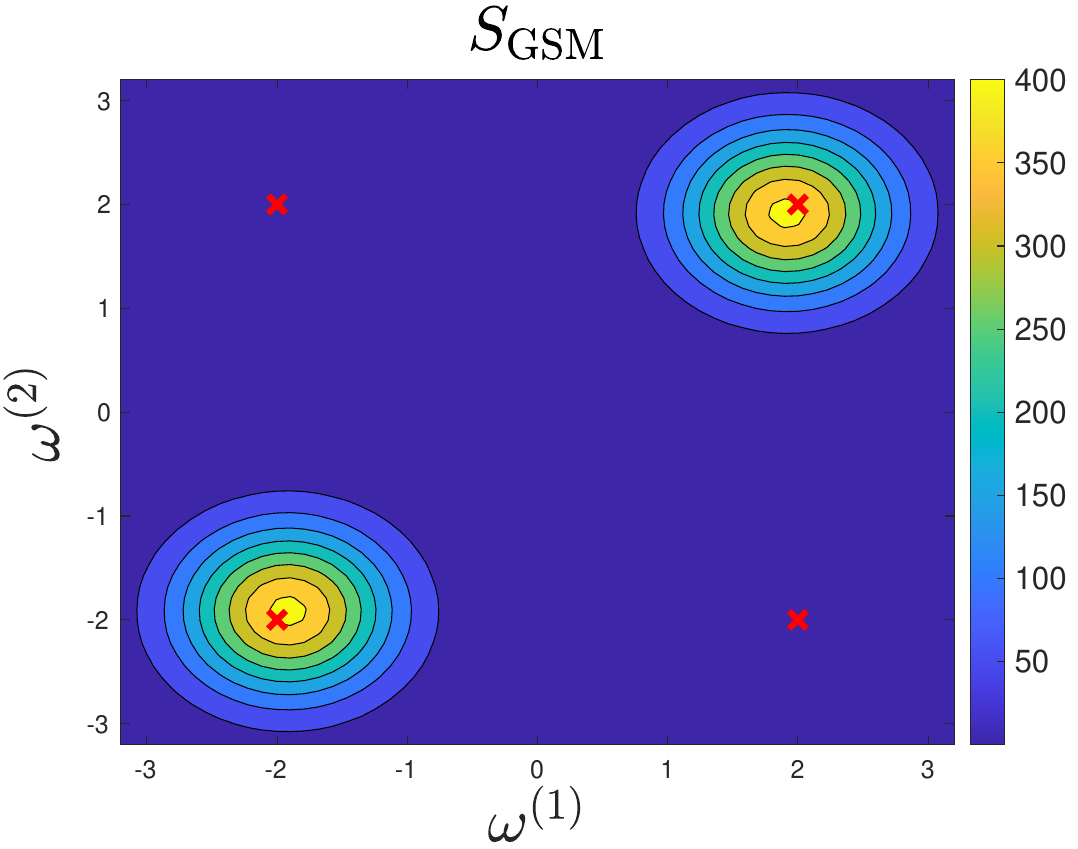}   \hfill 
\includegraphics[width=.24\textwidth]{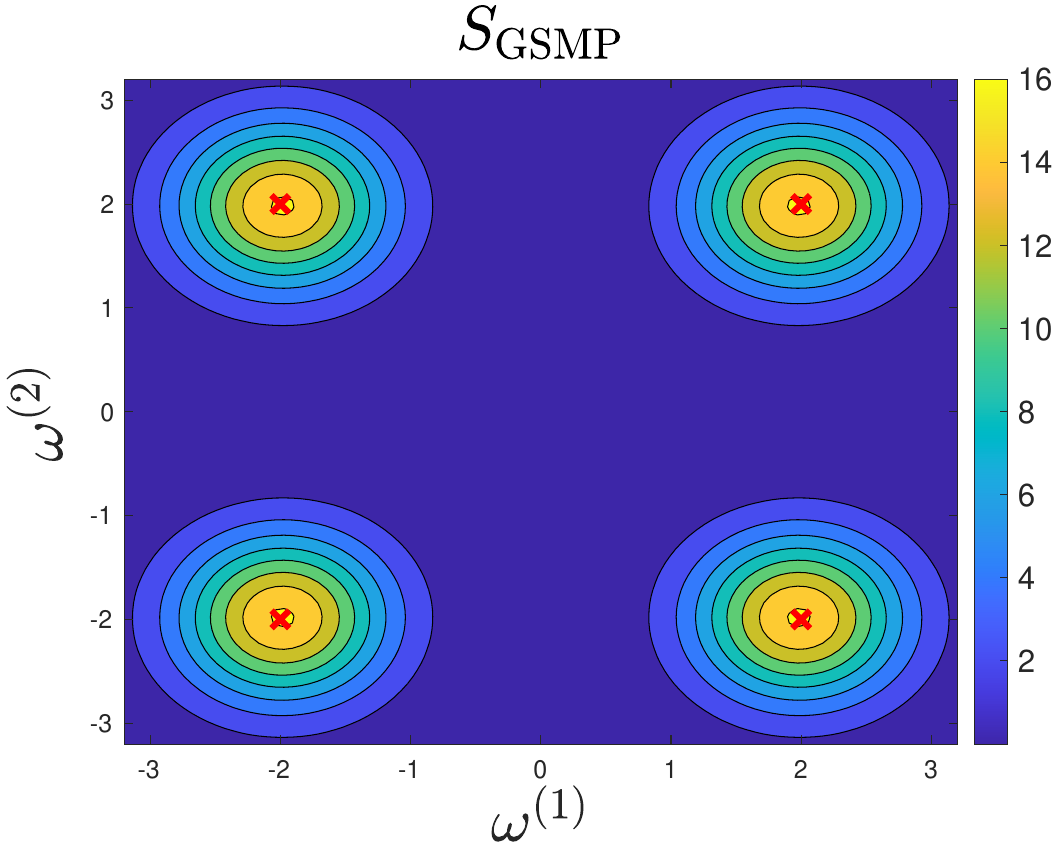}} 

\subfloat[Configuration 2]{\includegraphics[width=.24\textwidth]{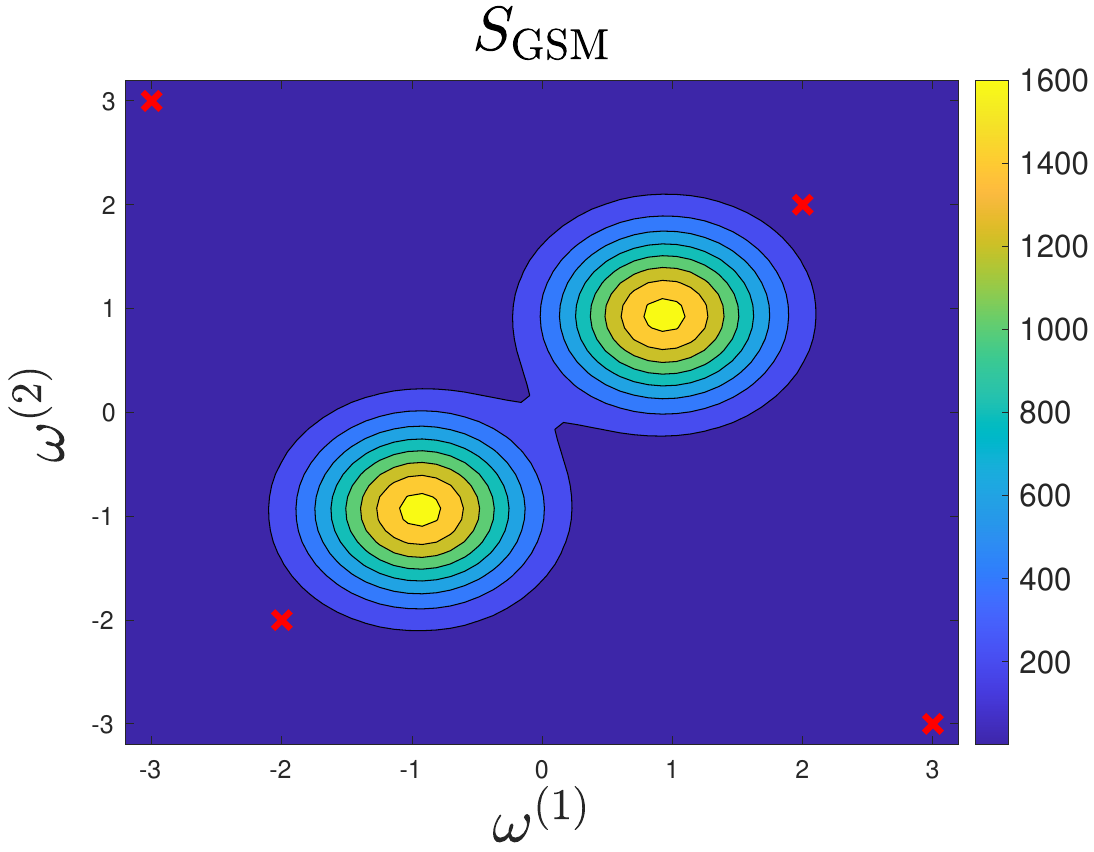}  \hfill 
\includegraphics[width=.24\textwidth]{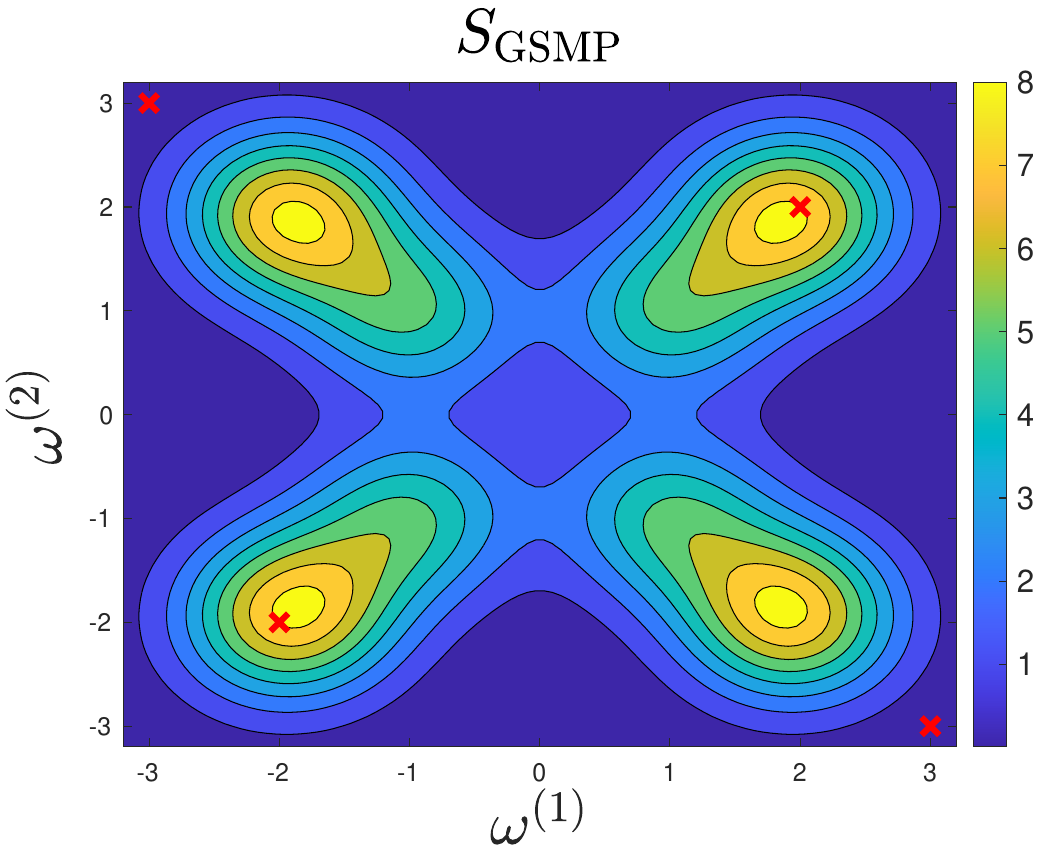}} 

\subfloat[Configuration 3]{\includegraphics[width=.24\textwidth]{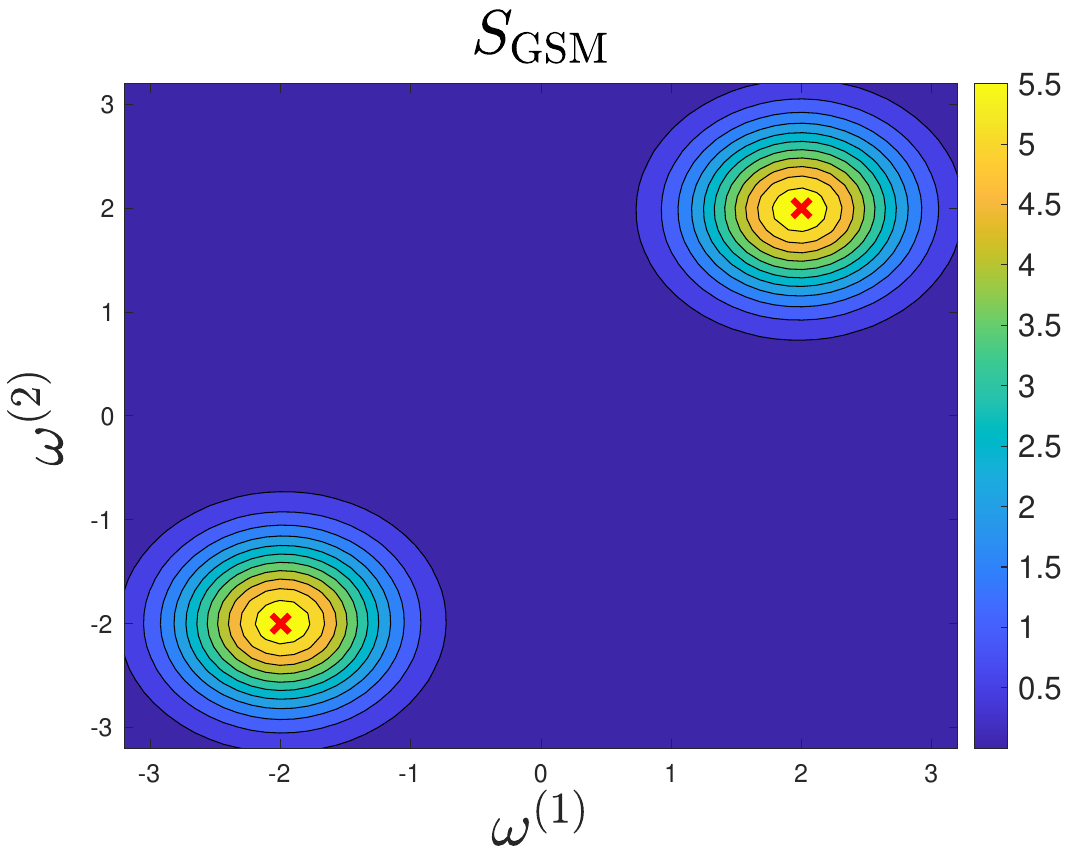}  \hfill 
\includegraphics[width=.24\textwidth]{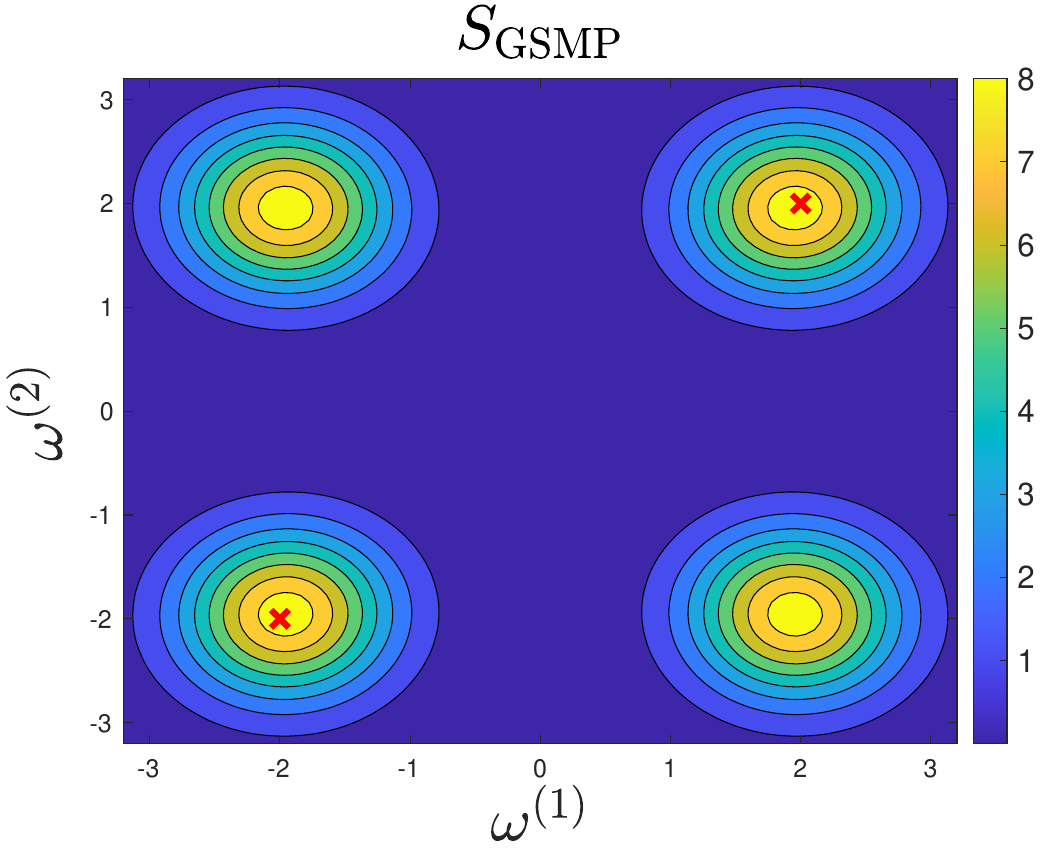}} 

\subfloat[Configuration 4]{\includegraphics[width=.24\textwidth]{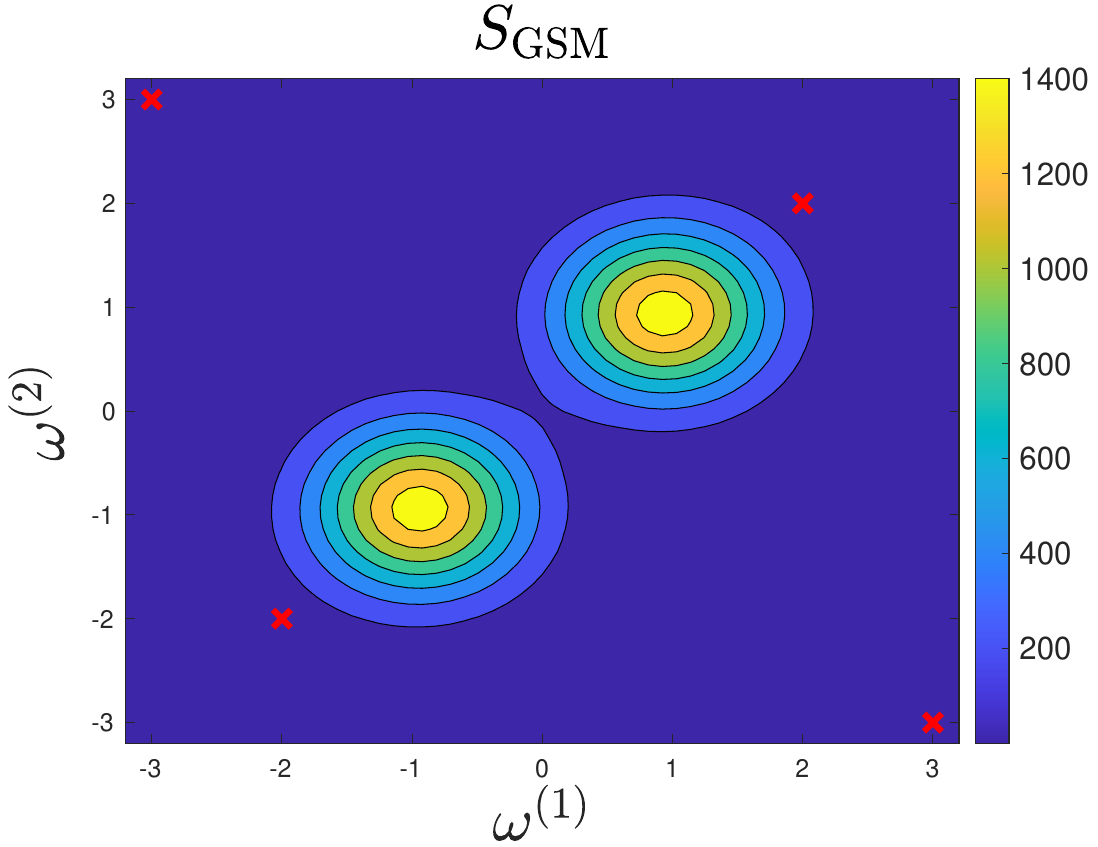}  \hfill 
\includegraphics[width=.24\textwidth]{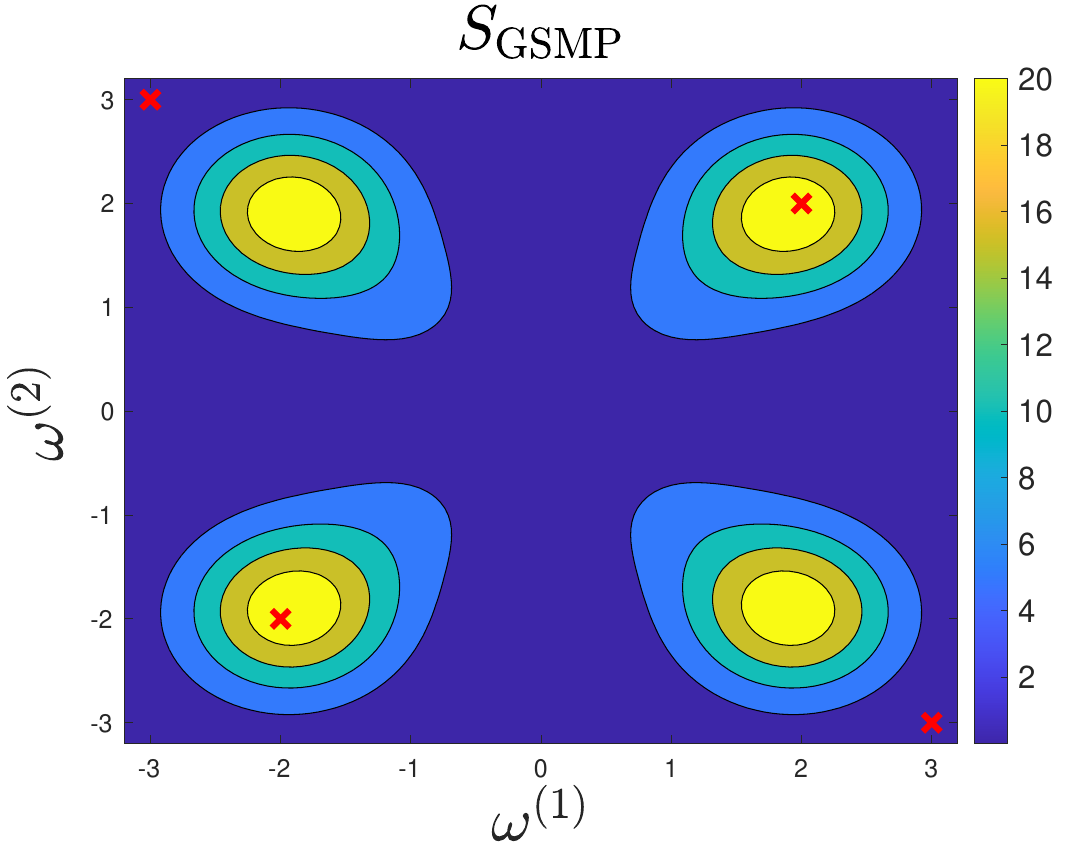}} 

\subfloat[Configuration 5]{\includegraphics[width=.24\textwidth]{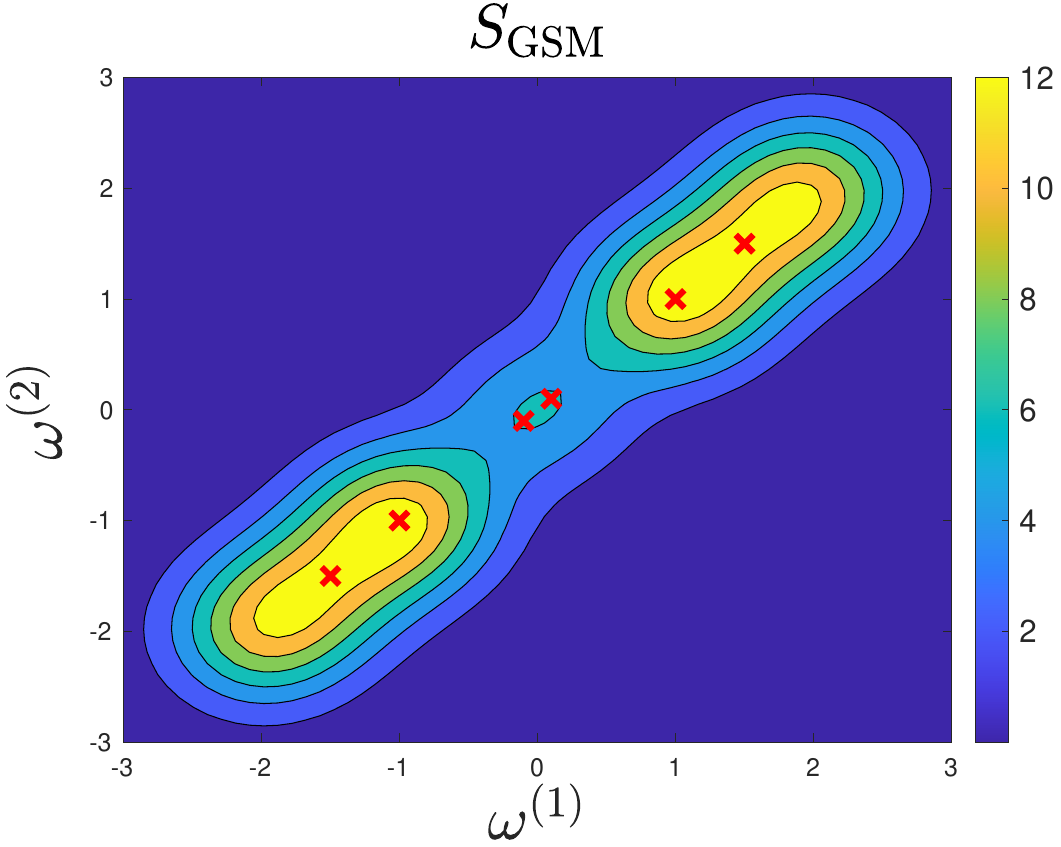}  \hfill 
\includegraphics[width=.24\textwidth]{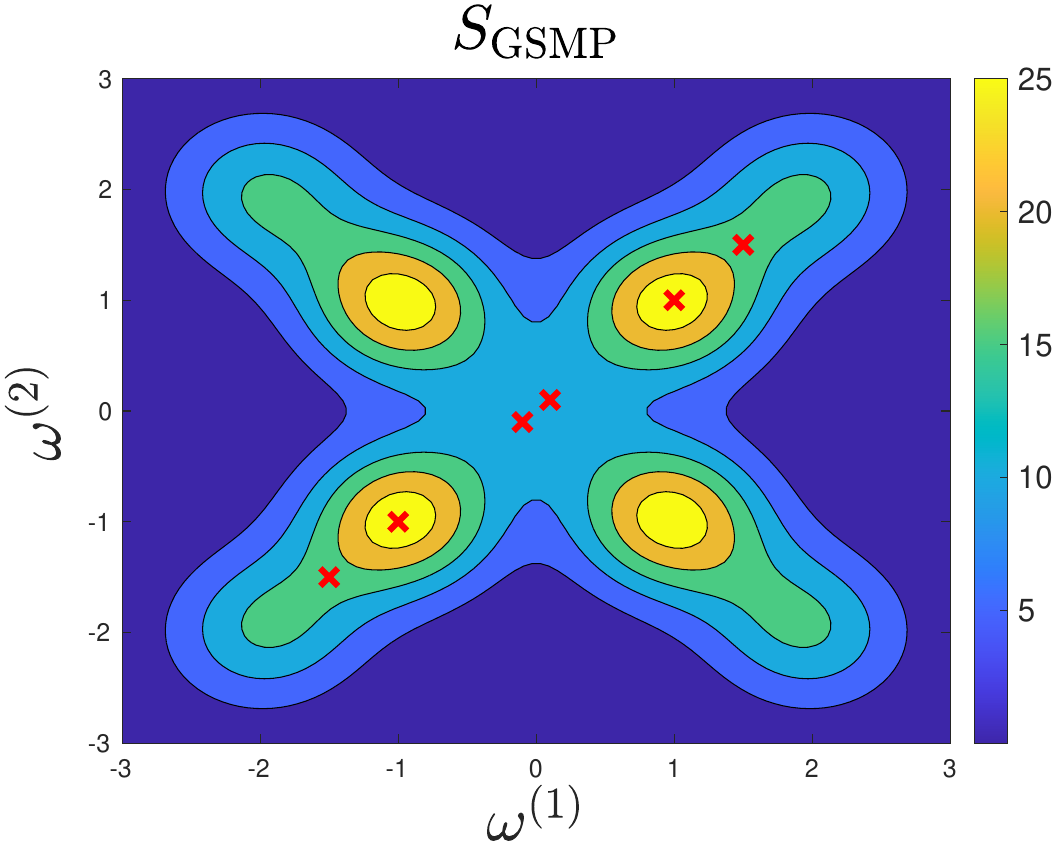}} 
    \caption{\label{fig:gsmgp-vs-gsm} The learned spectral density of multi-dimensional GSM kernel (left) versus the learned spectral density of GSMP kernel (right). The cross symbols represent the modes of the ground truth.}
\end{figure}

For the comparison, we set $Q = 50$ and sample the grid points from only the positive quadrant, that is, $\mu_q^{(p)}$ are uniformly sampled from $[0, 3]$ and the variance variables $v_q^{(p)}$ are set to $0.01$. 
The spectral density results of the learned kernels are depicted in Fig. \ref{fig:gsmgp-vs-gsm}, and the GP prediction results are presented in Table \ref{tab:gsmpVSgsm}.  

From the results of configurations 1, 2, and 4 in Fig.~\ref{fig:gsmgp-vs-gsm}, we observe that when the density modes of the underlying kernel lie along the main diagonal, the GSM kernel struggles to represent this density accurately. This is because the GSM kernel's fixed grid points support the density on the anti-diagonal but not the density on the main diagonal. In contrast, the GSMP kernel contains additional $2Q$ Gaussian densities, enabling it to capture the main diagonal's density effectively. This superior fitting of the underlying density translates into GSMP kernel's enhanced prediction performance, as corroborated by the results shown in Table \ref{tab:gsmpVSgsm}. If one wants to compensate for the performance deficiencies of the GSM kernel, it requires fixing additional $Q$ grid points, totaling $2Q$ grid points. Conversely, the GSMP kernel achieves comparable performance while only requiring $Q$ grid points. In configurations 3 and 5 of Fig.~\ref{fig:gsmgp-vs-gsm}, where the density pattern aligns with the anti-diagonal, the GSM kernel approximates this density adequately. On the other hand, the GSMP kernel not only captures the underlying anti-diagonal density but also introduces additional density on the main diagonal. However, the GP prediction performance using the GSMP kernel, as showcased in Table \ref{tab:gsmpVSgsm}, demonstrates a marginal difference from the performance achieved with the GSM kernel. This indicate that the GSMP kernel, with its additional Gaussian densities, still achieves comparable performance without the need for extra grid points, thus maintaining a lower computational complexity.

The above empirical findings highlight that the GSMP kernel not only effectively approximates the underlying kernel, thus ensuring good GP prediction performance, but notably diminishes the number of hyper-parameters. This reduction significantly mitigates both model and computational complexities, supporting our claim in Section \ref{subsec:gsmp-properties}.

\begin{table}[t!]
\centering
\caption{Prediction  mean squared error  (MSE) of the GP with GSMP kernel versus the GP with GSM kernel \cite{suwandiGaussianProcessRegression2022}}.
\label{tab:gsmpVSgsm}
\setlength{\tabcolsep}{2mm}{
\begin{tabular}{c| ccccc}
\toprule
Configuration &  1  & 2 & 3 &  4 & 5 \\ \midrule
GSM  & 0.0172  &0.5564  &\textbf{0.0046}  & 0.9431 &  \textbf{7.2646E-03} \\
GSMP & \textbf{0.0101}  &\textbf{0.1732}  &0.0063  & \textbf{0.3769} &  8.3232E-03\\
\bottomrule
\end{tabular}
}
\end{table}

\begin{table*}[t]
\centering
\caption{Details of the selected real datasets.
}
\label{tab:datasets}
\begin{tabular}{@{}rcccccccccccccc@{}}
\toprule
Dataset                & ECG & CO2  & Unemployment &  ALE  & CCCP  & Toxicity & Concrete & Wine & Water \\ \midrule
Training Size     & 680 & 481  & 380 & 80 & 1000      & 436      & 824   & 1279 & 5000 \\
Test Size & 20  & 20  & 20           & 20 & 250       & 110      & 206      & 320 & 100  \\
 Dimension                & 1   & 1  & 1            & 4  & 4         & 8        & 8        & 11 & 11  \\ \bottomrule
\end{tabular}
\end{table*}

\begin{table*}[!t]
\centering
\caption{Performance comparison between the proposed GSMPGP with SLIM-KL ($s = 4, N = 2, \Delta = 0.01$) and its competitors in terms of the prediction MSE. The lowest MSE value is highlighted in \textbf{bold}.  The percentage in the brackets indicates the sparsity level of the solution. The lower the value, the sparser the solution.}
\setlength{\tabcolsep}{4mm}{
\begin{tabular}{@{}rcccccccc@{}}
\toprule
Dataset      & GSMPGP & GSMGP & SMGP  & SSGP & SEGP    & LSTM    
& Informer   \\ \midrule
ECG          & \textbf{1.2E-02} \ (0.8\%)   & \textbf{1.2E-02} \ (0.8\%)      & 1.9E-02 \ (100\%) & 1.6E-01 & 1.6E-01 & 1.6E-01 
&  5.4E-02          \\
CO2          & \textbf{3.7E-01} \ (0.6\%)  &  6.2E-01 \ (1.0\%)              & 7.4E-01 \ (100\%) & 2.0E+02 & 1.5E+03 & 2.9E+02 
& 8.4E+01           \\
Unemployment & \textbf{2.0E+03} \ (2.8\%)  & 2.2E+03 \ (4.6\%)       & 7.7E+03 \ (100\%) & 2.1E+04 & 5.6E+05 & 1.7E+05 
& 3.8E+03           \\
ALE          & \textbf{2.5E-02} \ (0.24\%)  & 1.9E+00 \ (0.5\%)    & 3.7E-01 \ (100\%) & 5.8E-02 & 3.7E-02 & 3.4E-02 
& 1.8E-01           \\
CCCP         & \textbf{1.6E+01} \ (0.24\%)  &  1.9E+02 \ (0.24\%)  & 2.1E+05 \ (100\%) & \textbf{1.6E+01} & 1.7E+01 & 2.8E+02 
& 1.4E+05   \\
Toxicity     & \textbf{1.4E+00} \ (1.13\%)   &  2.7E+02 \ (1.13\%) & 5.3E+00 \ (100\%)  & 1.7E+00 & \textbf{1.4E+00} & 1.7E+00 
& 2.9E+00           \\
Concrete     & \textbf{5.9E+01} \ (0.13\%)   &  5.6E+02 \ (0.24\%) & 1.7E+03 \ (100\%)  & 3.5E+02 & 1.3E+02 & 1.4E+02 
&   2.0E+02         \\
Wine         & \textbf{4.6E-01} \ (0.09\%) &    7.4E+03 \ (0.09\%)   & 3.1E+01 \ (100\%) & 5.6E-01 & 4.4E+00 & 4.7E-01 
&   1.6E+00       \\
Water       & \textbf{1.6E-04} \ (1.27\%) & 1.7E-04 \ (1.0\%)     & 5.5E-03 \ (100\%)  & 2.7E-04 & \textbf{1.6E-04}  & 4.2E-04 
&  9.2E-02            \\
\bottomrule
\end{tabular}}
\label{tab:gsmgp-vs-competitors}
\end{table*}

\subsection{Prediction Performance}
\label{subsec:prediction}
This subsection showcases the prediction performance of the proposed GSMP kernel-based GP (GSMPGP) trained using SLIM-KL on nine real-world datasets. The dimensionality of these datasets varies from $1$ to $11$, with more details provided in Table \ref{tab:datasets}. The setup for training the GSMPGP is listed as follows.
The number of mixture components $Q$ is set to 500 for the one-dimensional case and $100 \times P$ for the multi-dimensional case. The mean variables $\mu_q^{(p)}$ are uniformly sampled from the frequency range mentioned in Section \ref{subsec:gsmp} and the variance variables $v_q^{(p)}$ are fixed to a small constant, $0.001$. The weights $\boldsymbol{\theta}$ and the dual variables $\{\boldsymbol{\lambda}_j\}_{j = 1}^N$ are initialized to zeros. The penalty parameter $\{\rho_j\}_{j = 1}^N$ is initialized to $10^{-10}$ and adaptively updated using the residual balancing strategy \cite{wohlbergADMMPenaltyParameter2017}. The noise variance parameter $\sigma_e^2$ is estimated using a cross-validation filter type method \cite{garciaRobustSmoothingGridded2010}.

For comparison, we consider six other models, the original GSM kernel-based GP (GSMGP) \cite{suwandiGaussianProcessRegression2022}, an SM kernel-based GP (SMGP) \cite{wilsonGaussianProcessKernels2013}, a sparse spectrum GP (SSGP) \cite{lazaro2010sparse}, a squared-exponential kernel-based GP (SEGP) \cite{rasmussenGaussianProcessesMachine2006}, a long short-term memory (LSTM) recurrent neural network \cite{hochreiterLongShortTermMemory1997}, 
and a very recent Transformer-based time series prediction model, Informer \cite{zhou2021informer}. The SMGP and SEGP model hyper-parameters are optimized using a gradient-descent based approach with the default parameter settings suggested in \href{https://people.orie.cornell.edu/andrew/code}{https://people.orie.cornell.edu/andrew/code}. The GSM and SM kernel use the same number of kernel components as the GSMP kernel ($Q = 500$ or $Q = 100 \times P$) for a fair comparison. The SSGP model has 500 basis functions with hyper-parameters determined using a conjugate-gradient method. The LSTM model has a standard architecture with 3 hidden layers, each with 100 hidden units while the Informer model follows the default configuration as specified in \href{https://github.com/zhouhaoyi/Informer2020/tree/main}{https://github.com/zhouhaoyi/Informer2020/tree/main}.

Table \ref{tab:gsmgp-vs-competitors} shows the prediction performance of various models as measured by the mean-squared-error (MSE). These results demonstrate that the proposed GSMPGP consistently outperform their competitors in prediction MSE across all datasets. Notably, GSMPGP outperforms the original GSMGP, particularly in multi-dimensional datasets, aligning with our findings in Section \ref{subsec:gsmp-vs-gsm}. While classical models like LSTM, and more recent models like Informer, require long time series data to learn the underlying patterns effectively, our GSMPGP model excels by building the correlations between series data and minimizing the negative marginal log-likelihood to optimize the kernel hyper-parameters. Furthermore, our GSMPGP model outperforms both SMGP and SSGP, as well as other GP models that use basic kernels like SEGP, highlighting its superior performance compared to competitors. In addition to superior prediction performance, our proposed method enhances scalability and communication-efficiency through distributed learning and quantization, as detailed in Sections \ref{subsec:scalability} and \ref{subsec:quantization}, respectively.

The improved prediction performance not only comes from the flexible model capacity of GSMPGP, but also benefits from the inherent sparsity in the solutions it generates. The sparsity level of the solutions, calculated as the ratio of nonzero elements to the total number of elements, is reported in Table \ref{tab:gsmgp-vs-competitors}. We note that both GSMPGP and GSMGP obtain sparse solutions, in comparison to SMGP. By pinpointing only the significant sub-kernels, GSMPGP and GSMGP avoid over-parameterization, yielding better generalization capability and lower prediction MSE.

\subsection{Scalability}
\label{subsec:scalability}
In this subsection, we aim to investigate the scalability of SLIM-KL through experiments with varying numbers of computing units and agents. We first consider varying $s \in \{1, 4, 10, 50, 100\}$, and evaluate the performance based on the prediction MSE and computation time (CT). Results in Table \ref{tab:varying-s-MSE} show that the prediction MSE does not always worsens as $s$ increases. In fact, for several datasets, an improved MSE can be achieved when using more than one computing unit. Moreover, Fig. \ref{fig:dsca_ct} highlights the advantage of using multiple computing units in reducing the overall CT, supporting the claim given in Remark \ref{remark:complexity}.

\begin{table}[!t]
\centering
\caption{Performance of SLIM-KL ($N = 2, \Delta = 0.01$) with varying number of computing units $s$, in terms of the prediction MSE. The lowest MSE value for each dataset is highlighted in \textbf{bold}.}
\begin{tabular}{@{}rccccc@{}}
\toprule
Dataset     & $s = 1$ & $s = 4$ & $s = 10$& $s = 50$& $s = 100$\\ \midrule
ECG          & \textbf{1.2E-02}                                       & \textbf{1.2E-02}                                       & \textbf{1.2E-02}                                        & 1.5E-02                                                 & 1.5E-02                                                  \\
CO2          & {4.9E-01}                                       & \textbf{3.7E-01}                                                & 6.2E-01                                                 & 1.1E+00                                                 & 6.6E-01                                                  \\
Unemployment & 2.3E+03                                                & \textbf{2.0E+03}                                       & 2.2E+03                                        & 2.2E+03                                        & 3.3E+03                                                  \\ 
ALE      & \textbf{2.5E-02}         & \textbf{2.5E-02}                                                            & \textbf{2.5E-02}                                                           & \textbf{2.5E-02}  & \textbf{2.5E-02}                                               \\
CCCP     & \textbf{1.6E+01} & \textbf{1.6E+01}                                                   & \textbf{1.6E+01}                                                   & \textbf{1.6E+01}  
& \textbf{1.6E+01}                                                \\ 
Toxicity & \textbf{1.4E+00} & \textbf{1.4E+00  }                                                         & \textbf{1.4E+00}                                                           &\textbf{ 1.4E+00} &   \textbf{1.4E+00}                                          \\
Concrete & \textbf{5.9E+01} & \textbf{5.9E+01}                                                  & \textbf{5.9E+01}                                                  & \textbf{5.9E+01} & \textbf{5.9E+01}
\\
Water & \textbf{1.6E-04} & \textbf{1.6E-04} & \textbf{1.6E-04} & \textbf{1.6E-04}  & \textbf{1.6E-04}  \\
Wine & \textbf{4.6E-01} & \textbf{4.6E-01}                                                    & \textbf{4.6E-01}                                                   & \textbf{4.6E-01}  & \textbf{4.6E-01} \\
\bottomrule
\end{tabular}
\label{tab:varying-s-MSE}
\end{table}

\begin{figure}[t!]
    \centering
    \includegraphics[width=1\linewidth]{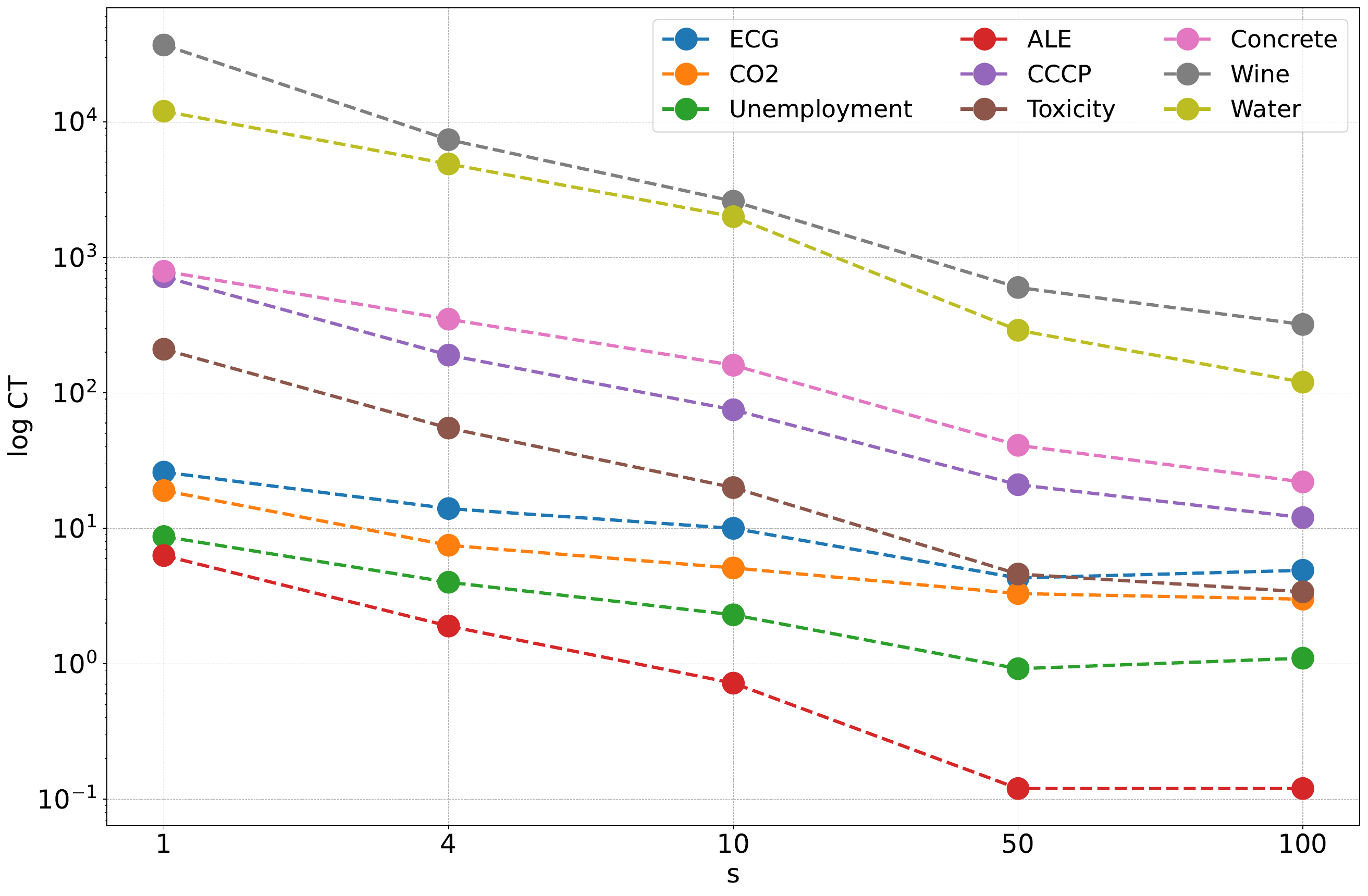}
    \caption{Total computation time (in log-scale) for one computing unit, with respect to different values of $s$.}
    \label{fig:dsca_ct}
\end{figure}

We also test the effect of varying the number of local agents on the prediction performance of SLIM-KL. In particular, we set $N \in \{2, 4, 8, 10\}$ and compare them with the centralized case which serves as the baseline. The results in Table \ref{tab:different-N-MSE} illustrate that SLIM-KL performs better on nearly all datasets when using multiple local agents compared to the centralized case. However, when $N$ is set to be relatively large compared to the size of the training set, the prediction performance degrades as each local agent has fewer data available for training, as evident, e.g., in the \textit{Unemployment} dataset. Hence, there is a trade-off between utilizing more distributed agents and ensuring each agent has sufficient data for training.

\begin{table}[!t]
\centering
\caption{Performance of SLIM-KL  ($s = 4, \Delta = 0.01$) with varying number of local agents $N$, in terms of the prediction MSE. The lowest MSE value for each dataset is highlighted in \textbf{bold}.}
\begin{tabular}{@{}rccccc@{}}
\toprule
Dataset      & Centralized & $N = 2$& $N = 4$ & $N = 8$ & $N = 10$ \\ \midrule
ECG          & \textbf{1.2E-02}                                            & \textbf{1.2E-02}                                            & 1.3E-02                                                     & 1.4E-02                                                     & 2.2E-02                                                      \\
CO2          & 4.5E-01                                                     & \textbf{3.7E-01}                                            & 5.3E-01                                                     & 8.0E-01                                                     & 5.7E-01                                                      \\
Unemployment & 2.3E+03                                                     & \textbf{2.0E+03}                                            & 4.9E+03                                                     & 8.2E+03                                                     & 9.1E+03                                                      \\ 
ALE      & \textbf{2.5E-02} & \textbf{2.5E-02} & \textbf{2.5E-02} & \textbf{2.5E-02} & \textbf{2.5E-02}                                      \\
CCCP     & \textbf{1.6E+01} & \textbf{1.6E+01}  & \textbf{1.6E+01}  & \textbf{1.6E+01}  & \textbf{1.6E+01}             \\ 
Toxicity & \textbf{1.4E+00} & \textbf{1.4E+00} & 1.8E+00 & 1.7E+00 & 1.8E+00                                         \\
Concrete & 8.3E+01 & \textbf{5.9E+01} & 8.4E+01 & 8.4E+01  & 8.5E+01  \\
Water & \textbf{1.6E-04} & \textbf{1.6E-04} & \textbf{1.6E-04} & \textbf{1.6E-04} & 1.7E-04  \\
Wine & \textbf{4.6E-01} & \textbf{4.6E-01} & 4.7E-01 & 4.9E-01 & 5.3E-01 \\
\bottomrule
\end{tabular}
\label{tab:different-N-MSE}
\end{table}

To further evaluate the scalability of SLIM-KL, we test its performance on a large dataset. Specifically, we consider the \textit{CCCP} dataset and increase the training set size $|\mathcal{D}|$ to 9500 and the test set size $|\mathcal{D}_*|$ to 68. We set $N = 10$ and $s = 4$ while keeping the other experimental setup unchanged. Remarkably, we observed that SLIM-KL continues to perform well even with a dataset nearly ten times larger. The prediction MSE remains low at 1.6E+01, highlighting the algorithm's ability to handle big data.

\subsection{Effect of Quantization}
\label{subsec:quantization}
\begin{figure}[t!]
    \centering
    \includegraphics[width=\linewidth]{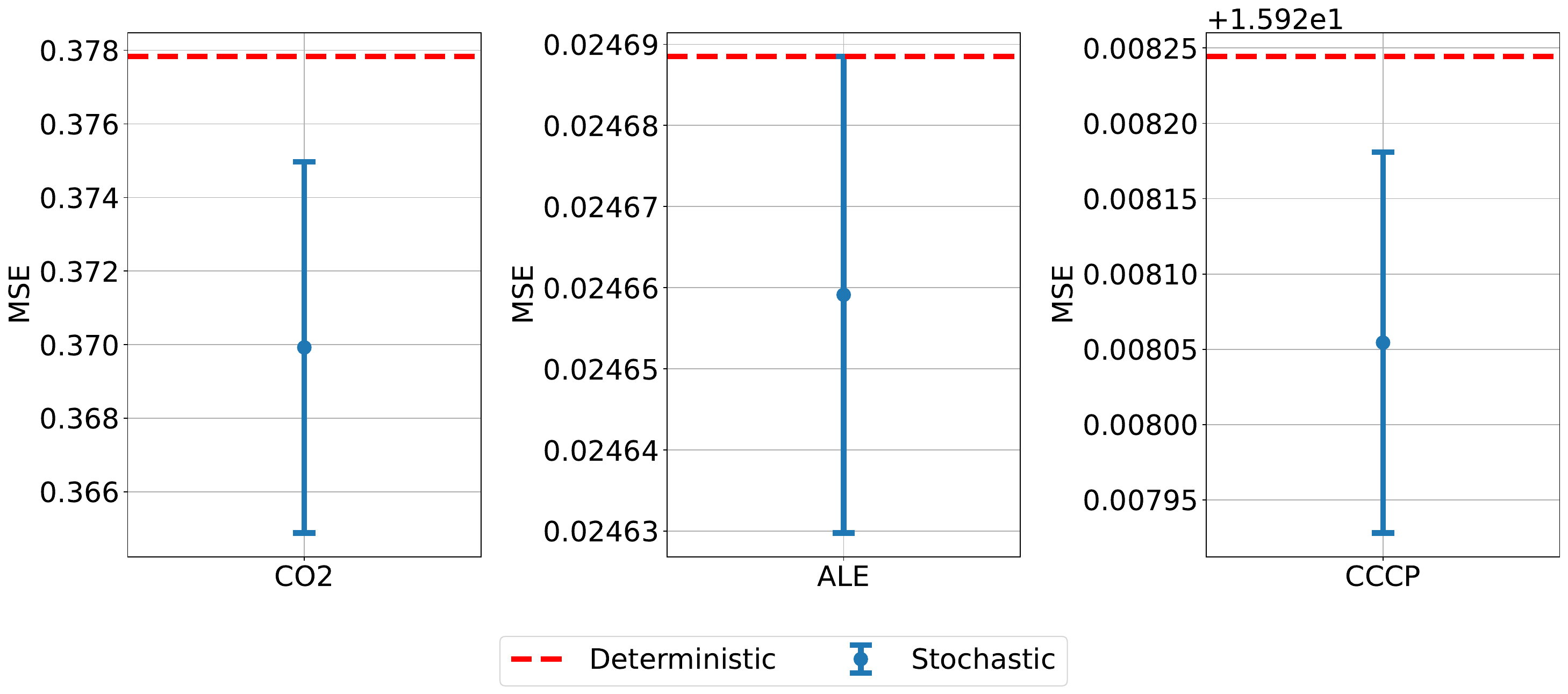}
    \caption{Performance comparison of SLIM-KL under stochastic quantization versus deterministic quantization, with $\Delta = 0.01$. The blue dots with vertical error bars indicate the mean MSE plus-minus two standard deviations when using the stochastic quantization, while the red dashed lines represent the MSE when using the deterministic quantization.}
    \label{fig:quantization-comparison}
\end{figure}

This subsection studies the effect of quantization on the performance of SLIM-KL through ablation experiments. We compare two different choices of quantization scheme: stochastic quantization and deterministic quantization. Due to the random nature of the stochastic quantization, we repeat each experiment 5 times and report the mean MSE with corresponding error bars. Fig. \ref{fig:quantization-comparison} reveals that the stochastic quantization scheme achieves a consistently lower MSE than deterministic quantization, despite its variability. The error bars, denoting plus-minus two standard deviations, suggest that while the MSE for stochastic quantization varies, it does so within a bounded variance, indicating a reliable performance with potential benefits from its inherent randomness.

We further studied the performance of SLIM-KL under different quantization resolutions $\Delta \in \{0.001, 0.01, 0.1, 1, 5\}$ and compared them with the case without quantization. To evaluate the performance, we empirically computed the prediction MSE and the number of bits required to transmit the local hyper-parameters. Assuming double precision, the size of the communication overhead for transmitting a normal vector $\boldsymbol{\theta}^t \in \mathbb{R}^Q$ is $B = 64Q$ bits, while for the quantized vector $\mathcal{Q}(\boldsymbol{\theta}^t)$ it is $B_\mathcal{Q} = Q \log_2 \left( (\theta_\mathrm{max}^t - \theta_\mathrm{min}^t)/\Delta + 1 \right)$ bits. Thus, the saving ratio in transmission bits when using quantization can be computed as: $B/B_\mathcal{Q} = 64/\log_2 \left( (\theta_\mathrm{max}^t - \theta_\mathrm{min}^t)/\Delta + 1 \right)$.

Results in Table \ref{tab:different-delta} demonstrate that quantization does not necessarily degrade prediction performance. In fact, on the \textit{CO2} dataset, setting $\Delta = 1$ yields the lowest prediction MSE while saving $4.8\times$ fewer bits on average, as depicted in Fig.~\ref{fig:saving_ratio_sep}. This improved performance corroborates the design of SLIM-KL, where the enhanced sparsity can be beneficial in guiding the algorithm towards a more favorable optimum, as elaborated in Remark \ref{remark:slim-kl}. Overall, quantization demonstrates remarkable communication cost savings of over 200 times in the \textit{Water} dataset while maintaining promising prediction performance. 

\begin{figure}[t]
    \centering
    \includegraphics[width=\linewidth]{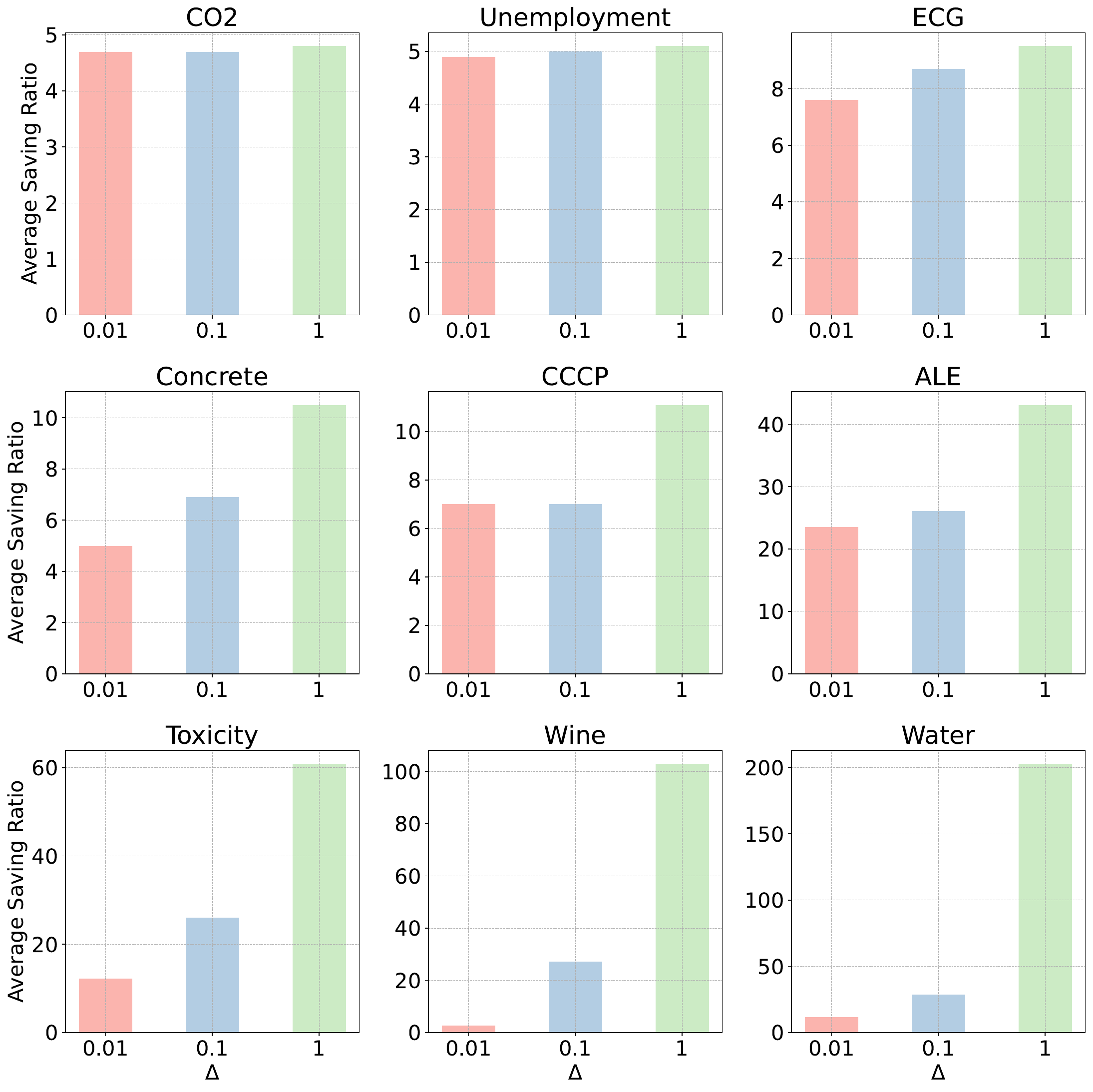}
    \caption{Average saving ratio in transmitting the local hyper-parameters when using quantization versus without quantization, with respect to different quantization resolution $\Delta$.}
    \label{fig:saving_ratio_sep}
\end{figure}

\begin{remark}
A lower quantization resolution (e.g., $\Delta = 0.01$) offers higher precision but requires more bits, increasing communication overhead. In contrast, a higher resolution (e.g., $\Delta = 1$) reduces data transmission size, but may degrade performance due to information loss. We found that a resolution of 0.1 strikes a practical balance between model performance and communication cost.
\end{remark}

\begin{table*}[!t]
\centering
\caption{Performance of SLIM-KL ($s = 4, N = 2$) with and without quantization in terms of the prediction MSE. The lowest MSE value is highlighted in \textbf{bold}.}
\label{tab:different-delta}
\begin{tabular}{rcccccc}
\toprule
\multirow{2}{*}{Dataset} & \multirow{2}{*}{\begin{tabular}[c]{@{}c@{}}Without \\ Quantization\end{tabular}} & \multicolumn{5}{c}{With Quantization}  \\ 
& & \multicolumn{1}{l}{$\Delta = 0.001$} & \multicolumn{1}{l}{$\Delta = 0.01$} & \multicolumn{1}{l}{$\Delta = 0.1$} & \multicolumn{1}{l}{$\Delta = 1$} & \multicolumn{1}{l}{$\Delta = 5$} \\ \midrule
ECG          & \textbf{1.2E-02} & \textbf{1.2E-02} & \textbf{1.2E-02} & 2.7E-02 & 1.6E-01 & 3.9E-01 \\
CO2          & 3.7E-01 & 3.7E-01 & 3.7E-01 & 6.5E-01 & \textbf{2.7E-01} & 6.3E-01 \\
Unemployment & \textbf{2.0E+03} & \textbf{2.0E+03} & \textbf{2.0E+03} & \textbf{2.0E+03} & \textbf{2.0E+03} & 7.2E+03 \\ 
ALE          & \textbf{2.5E-02} & \textbf{2.5E-02} & \textbf{2.5E-02} & \textbf{2.5E-02} & \textbf{2.5E-02} & 0.9E-01 \\
CCCP         & \textbf{1.6E+01} & \textbf{1.6E+01} & \textbf{1.6E+01} & \textbf{1.6E+01} & \textbf{1.6E+01} & 1.7E+01 \\ 
Toxicity     & \textbf{1.4E+00} & \textbf{1.4E+00} & \textbf{1.4E+00} & \textbf{1.4E+00} & \textbf{1.4E+00} & 1.5E+00 \\
Concrete     & \textbf{5.9E+01} & \textbf{5.9E+01} & \textbf{5.9E+01} & \textbf{5.9E+01} & \textbf{5.9E+01} & \textbf{5.9E+01} \\
Water        & \textbf{1.6E-04} & \textbf{1.6E-04} & \textbf{1.6E-04} & \textbf{1.6E-04} & 4.4E-01 & 4.4E-01 \\
Wine         & \textbf{4.6E-01} & \textbf{4.6E-01} & \textbf{4.6E-01} & \textbf{4.6E-01} & \textbf{4.6E-01} & 4.8E-01 \\
\bottomrule
\end{tabular}
\end{table*}

\subsection{Long Time Series Forecasting}
\label{subsec:lstf}
In this subsection, we extend the evaluation of our proposed method to larger-scale datasets for real world long time series forecasting tasks. We assess the performance using five popular real world datasets: \textit{Weather}\footnote{\url{https://www.bgc-jena.mpg.de/wetter/}} and four \textit{ETT}\footnote{\url{https://github.com/zhouhaoyi/ETDataset}} datasets (\textit{ETTh1, ETTh2, ETTm1, ETTm2}). The details of these datasets can be found in \cite{zhou2021informer}. For the GSMPGP model, we optimized the hyper-parameters using SLIM-KL with $s=100$ computing units, $N=100$ agents, and a quantization resolution of $\Delta = 0.01$. For the other models, we followed the configurations reported in \cite{lin2023segrnn}. We compare our method against several recent state-of-the-art (SOTA) models, including SegRNN \cite{lin2023segrnn}, PatchTST \cite{nie2023time}, MICN \cite{wang2022micn}, and TiDE \cite{das2024longterm}.

The results in Table \ref{tab:lstf} reveal that our GSMPGP model ranked in the top two positions for four out of five datasets, securing first place in two instances. This performance highlights its superiority over other baseline models. Specifically, GSMPGP demonstrated exceptional performance on the \textit{ETT} datasets, nearly achieving SOTA results in terms of MSE. Although there was a slight decrease in performance on the \textit{Weather} dataset, GSMPGP remained competitive compared to other models. These results also indicate that, as we increase the number of agents, SLIM-KL scales effectively, maintaining competitive performance even with a large number of distributed agents. We further compare SLIM-KL with a recent distributed multiple kernel learning method, DOMKL \cite{hong2023domkl}, as well as two random feature-based multiple kernel learning methods, CoKle and BoKle \cite{hong2024omkl}. The results summarized in Table  demonstrate that SLIM-KL outperforms the other multiple kernel learning methods in terms of prediction MSE.

\begin{table*}[!t]
    \centering
    \caption{Performance comparison on long time series forecasting tasks. The lowest MSE value is highlighted in \textbf{bold}, and the second lowest is \underline{underlined}. The results of the SOTA models are reported from \cite{lin2023segrnn}.}
    \label{tab:lstf}
    \begin{tabular}{@{}rccccccc@{}}
    \toprule
    Dataset & GSMPGP (ours) & PatchTST (2023) & SegRNN (2023) & MICN (2023) & TiDE (2023) \\
    \midrule
    ETTh1 & \textbf{0.431} & 0.447 & \underline{0.434} & 0.770 & 0.454 \\
    ETTh2 & \underline{0.382} & \textbf{0.379} & 0.394 & 0.956 & 0.419 \\
    ETTm1 & \textbf{0.407} & 0.416 & \underline{0.410} & 0.481 & 0.413 \\
    ETTm2 & \underline{0.341} & 0.362 & \textbf{0.330} & 0.502 & 0.352 \\
    Weather & 0.313 & 0.314 & \textbf{0.310} & \underline{0.311} & 0.313 \\
    \bottomrule
    \end{tabular}
\end{table*}

\begin{table}[!t]
    \centering
    \caption{Performance comparison on long time series forecasting tasks. The lowest MSE value is highlighted in \textbf{bold}.}
    \label{tab:lstf_additional}
    \begin{tabular}{@{}rccccccc@{}}
    \toprule
    Dataset & SLIM-KL & DOMKL & CoKle & BoKle \\
    \midrule
    ETTh1 & \textbf{0.431} & 0.521 & 0.451 & 0.441 \\
    ETTh2 & \textbf{0.382} & 0.487 & 0.391 & 0.395 \\
    ETTm1 & \textbf{0.407} & 0.463 & 0.423 & 0.410 \\
    ETTm2 & \textbf{0.341} & 0.401 & 0.371 & 0.359 \\
    Weather & \textbf{0.313} & 0.372 & 0.348 & 0.330 \\
    \bottomrule
    \end{tabular}
\end{table}

\section{Conclusion}
\label{sec:conclusion}
In this paper, we proposed a novel GP kernel and a generic sparsity-aware distributed learning framework for linear multiple kernel learning. Experiments showed that the GSMP kernel, with a modest number of hyper-parameters, retains good approximation capability for multi-dimensional data. The GSMGP with SLIM-KL exhibited superior prediction performance over competitors across diverse datasets. Scalability tests revealed that SLIM-KL maintains and improves prediction performance with increased computing units and agents, demonstrating the framework's effectiveness for distributed environments. Furthermore, incorporating a stochastic quantization scheme significantly reduced communication overhead while preserving performance. These results demonstrate how the proposed method expands the capabilities of GPs to handle large-scale, multi-dimensional datasets and broadens their applicability across various engineering tasks.

\appendices
\section{Spectral Density of GSMP Kernel}
\label{appendix:density-gsmp}
We prove Theorem \ref{thm:GSMP_GMM}. By taking the Fourier transform of the GSMP kernel given in Eq. (\ref{eq:gsmp}), its spectral density can be obtained as,
\begin{equation}
\begin{aligned}
    & S_{\mathrm{GSMP}}(\boldsymbol{\omega}) \\
    & = \int_{\mathbb{R}^P} k(\bm{\tau}) \exp \left\{- j 2 \pi \bm{\tau}^\top \bm{\omega}  \right\} d \bm{\tau} \\
    &= \int_{\mathbb{R}^P} \sum_{q=1}^{Q} \theta_{q} \prod_{p=1}^{P} k_q\left(\tau^{(p)}\right) \exp \left\{- j 2 \pi \bm{\tau}^\top \bm{\omega}  \right\} d \bm{\tau} \\
    &= \sum_{q=1}^{Q} \theta_{q} \prod_{p=1}^{P} \underbrace{ \left[ \int_{\mathbb{R}} k_q\left(\tau^{(p)}\right) \exp \left\{ -j 2 \pi \tau^{(p)} \omega^{(p)} \right\} d\tau^{(p)} \right]}_{\text{Fourier transform of $k_q\left(\tau^{(p)}\right):$} \ \mathcal{F}\left[k_q\left(\tau^{(p)}\right)\right]}
\end{aligned}
\end{equation}
Subsituting the definition of $k_q\left(\tau^{(p)}\right)$ yields,
\begin{equation}
\begin{aligned}
    & \mathcal{F}\left[k_q\left(\tau^{(p)}\right)\right]  \\
    & = \mathcal{F}\left[\exp \left\{-2 \pi^{2} {\tau^{(p)^2}} v_{q}^{{(p)}}\right\} \cos \left(2 \pi \tau^{(p)} \mu_{q}^{(p)}\right)\right] \\
    &= \mathcal{F}\left[\exp \left\{-2 \pi^{2} {\tau^{(p)^2}} v_{q}^{{(p)}}\right\} \right] * \mathcal{F}\left[ \cos \left(2 \pi \tau^{(p)} \mu_{q}^{(p)}\right)\right] \\
    &= \frac{1}{\sqrt{{2\pi v_q^{(p)}}}} \exp \left\{\!-\frac{\omega^{{(p)}^2}}{2 v_q^{(p)}} \!\right\} * \frac{1}{2}\big(\delta(\omega^{(p)} \!-\! \mu_q^{(p)}) \\ & \hspace{2ex} + \, \delta(\omega^{(p)}+\mu_q^{(p)})\big) \\
    &= \frac{1}{2} \left[ \frac{1}{\sqrt{{2\pi v_q^{(p)}}}} \exp \left\{-\frac{\left(\omega^{(p)} - \mu_q^{(p)}\right)^2}{2 v_q^{(p)}} \right\} \right. \\ 
    & \hspace{4.5ex} \left. + \frac{1}{\sqrt{{2\pi v_q^{(p)}}}} \exp \left\{-\frac{\left(\omega^{(p)} + \mu_q^{(p)}\right)^2}{2 v_q^{(p)}} \right\} \right] \\
    &= \frac{1}{2} \left[ \mathcal{N}\left(\omega^{(p)}\mid \mu_q^{(p)}, v_q^{(p)}\right) + \mathcal{N}\left(\omega^{(p)}\mid - \mu_q^{(p)}, v_q^{(p)}\right) \right]
\end{aligned}
\end{equation}
Therefore, the spectral density of the GSMP kernel is given by,
\begin{equation}
\begin{aligned}
    &S_{\mathrm{GSMP}}(\boldsymbol{\omega}) = \sum_{q=1}^{Q} \theta_{q} \prod_{p=1}^{P} \mathcal{F}\left[k_q\left(\tau^{(p)}\right)\right] \\
    & = \frac{1}{2^P} \sum_{q=1}^{Q} \theta_{q} \prod_{p=1}^{P} \left[ \mathcal{N} \! \left(\omega^{(p)}\mid \mu_q^{(p)}, v_q^{(p)} \!\right) \!+\! \mathcal{N} \! \left(\omega^{(p)}\mid \!-\! \mu_q^{(p)}, v_q^{(p)} \!\right) \right]
\end{aligned}
\end{equation}
which is a Gaussian mixture.

\section{Proof of Theorem \ref{thm:sparsity}}
\label{appendix:sparsity}
\begin{proof}
The proof is an extension of \cite[Theorem 2]{wipf2004SBL}. We first state the following lemma that is necessary for the final result.

\begin{lemma}[Section IV.A \cite{wipf2004SBL}]
    \label{lemma:constant}
    The term $\boldsymbol{y}^\top \bm{C}^{-1}(\bm{\theta})\bm{y}$ equals a constant $C$ over all $\btheta$ satisfying the $n$ linear constraints $ \boldsymbol{A}\bm{\theta} = \boldsymbol{b}$ where,
     \begin{align}
        \label{eq:A}
        \boldsymbol{A} \triangleq [\boldsymbol{v}_1, \boldsymbol{v}_2, \ldots, \boldsymbol{v}_Q], \ \ 
        \boldsymbol{b} \triangleq \boldsymbol{y} - \sigma_e^2 \boldsymbol{u},
    \end{align}
    with $\boldsymbol{v}_q \triangleq \boldsymbol{K}_q \boldsymbol{u}$ and $\boldsymbol{u}$ is any fixed vector such that $\boldsymbol{y}^\top \boldsymbol{u} = C$. 
\end{lemma}

Consider the following optimization problem,
\begin{align}
    \label{eq:opt-problem}
    &\arg \min_{\btheta} \, \log \det \left(\bm{C}(\bm{\theta})\right) \nonumber \\
    \enspace &\operatorname{s.t.} \,\boldsymbol{A}\boldsymbol{\btheta} = \boldsymbol{b}, \btheta \geq \bm{0},
\end{align}
where $\boldsymbol{A}$ and $\boldsymbol{b}$ are as defined in Lemma \ref{lemma:constant}. Under the GSMP kernel constraint, we define $\boldsymbol{u} = \bm{C}^{-1}(\bm{\theta}) \boldsymbol{y}$, and deduce $\boldsymbol{y} - \sigma_e^2 \boldsymbol{u} = \sum_{q = 1}^Q \theta_q \boldsymbol{K}_q \boldsymbol{u}$. 

According to Lemma \ref{lemma:constant}, the above constraints hold $\bm{y}^\top  \bm{C}^{-1}(\bm{\theta})\bm{y}$ constant on a closed, bounded convex polytope. Hence, we focus on minimizing the second term of $l(\btheta)$ in Eq. (\ref{eq:DCP}) while holding the first term constant to some $C$. It follows that any local optimum of $l(\btheta)$ is also a local optimum of Eq. (\ref{eq:opt-problem}) with $C = \boldsymbol{y}^\top \boldsymbol{u}$. In \cite{luenbergerLinearNonlinearProgramming1984}, it is further established that all minima of Eq. (\ref{eq:opt-problem}) occur at extreme points and these extreme points are equivalent to basic feasible solutions, i.e., solutions with at most 
$n$ nonzero values. This implies that all local minima of Eq. (\ref{eq:opt-problem}) must be achieved at sparse solutions.
\end{proof}

\section{Local Problem Reformulation}
\label{appendix:reformulation}
Since the term $\tilde{l}_i(\bzeta_{j, i}, \bzeta_j^\eta)$ in Eq. (\ref{eq:dsca_problem}) contains a matrix fractional component, we can introduce an auxiliary variable $z$ such that, $\bm{y}^\top \bm{C}^{-1}(\bzeta_{j, i}, \bzeta_{j, -i}^\eta) \bm{y} \leq z$, and use the Schur complement condition \cite{boydConvexOptimization2004} to obtain an semi-definite programming (SDP) \cite{vandenbergheSemidefiniteProgramming1996} formulation of the problem:
\begin{align} \label{eq:dsca_sdp}
        &\min_{\bzeta_{j, i}, z} \, z - \nabla_{\bzeta_{j, i}} h(\bzeta_j^\eta)^\top \bzeta_{j, i} + \left(\boldsymbol{\lambda}_{j, i}^{t}\right)^{\top}\left(\bzeta_{j, i}- \btheta_i^{t+1}\right) 
        \nonumber \\ 
        &\hspace{4.5ex} + \frac{\rho_j}{2}\left\|\bzeta_{j, i}-\btheta_i^{t+1}\right\|_{2}^{2} \nonumber \\
        &\text{s.t.} \, \enspace \begin{bmatrix}
    \bm{C}(\bzeta_{j, i}, \bzeta_{j, -i}^\eta) & \bm{y} \\
    \bm{y}^\top & z
    \end{bmatrix}
    \succeq \bm{0}, \enspace \bzeta_{j, i} \geq \bm{0}.
\end{align}
Due to  the fact that $\bm{C}(\bzeta_{j, i}, \bzeta_{j, -i}^\eta)$ is a sum of positive semi-definite terms, the SDP can be reformulated by introducing $Q/s + 1$ rotated quadratic cone constraints \cite{loboApplicationsSecondorderCone1998}: 
\begin{align} \label{eq:cone_dsca}
        &\min_{\bzeta_{j, i}, \bm{z}, \bm{w}} \, (\bm{1}^\top \bm{z}) - \nabla_{\bzeta_{j, i}} h(\bzeta_j^\eta)^\top \bzeta_{j, i} + \left(\boldsymbol{\lambda}_{j, i}^{t}\right)^{\top}\left(\bzeta_{j, i}- \btheta_i^{t+1}\right) 
        \nonumber \\ 
        &\hspace{7ex} + \frac{\rho_j}{2}\left\|\bzeta_{j, i}-\btheta_i^{t+1}\right\|_{2}^{2} \nonumber \\
        &\text{s.t.} \, \enspace \|\bm{w}_0\|_2^2 \leq z_0 \nonumber ,\\
        & \hspace{4.24ex} \|\bm{w}_j\|_2^2 \leq \bzeta_{j, k  + (i - 1)Q/s} z_k, \enspace k \in \{ 1,2, \ldots, Q/s\} \nonumber ,\\
        & \hspace{4.5ex} \bm{y} = \sum_{k = 1}^{Q/s} \bm{L}_{k + (i - 1)Q/s} \bm{w}_k \nonumber \\
        & \hspace{6ex} + \left( \sum_{n \neq i}\sum_{k = 1}^{Q/s} \bzeta_{j, k  + (n - 1)Q/s}^\eta \bm{K}_{k  + (n - 1)Q/s} + \sigma_e^2 \bm{I}_n \right)^{\frac{1}{2}} \! \bm{w}_0 \nonumber ,\\
        & \hspace{4.75ex} \bzeta_{j, i} \geq \bm{0}, \enspace \bm{z} \geq \bm{0}.
\end{align}
where $\bm{z} = [z_0, z_1, \ldots, z_{Q/s}] \in \mathbb{R}^{Q/s + 1}$, $\bm{w}_0 \in \mathbb{R}^n, \bm{w}_k \in \mathbb{R}^{n_j}$ for $ k \in \{ 1, 2, \ldots, Q/s \}$, and $\bm{L}_k \in \mathbb{R}^{n \times n_k}$ is a low-rank kernel matrix approximation factor which can be obtained using random Fourier features \cite{rahimiRandomFeaturesLargeScale2007} or Nyström approximation \cite{williamsUsingNystromMethod2000}. This reformulation enables us to further express the problem as a second-order cone programming (SOCP) problem \cite{loboApplicationsSecondorderCone1998}:
\begin{align} \label{eq:socp_d2sca}
        &\min_{\bzeta_{j, i}, \bm{z}, \bm{w}} \, (\bm{1}^\top \bm{z}) - \nabla_{\bzeta_{j, i}} h(\bzeta_j^\eta)^\top \bzeta_{j, i} + \left(\boldsymbol{\lambda}_{j, i}^{t}\right)^{\top}\left(\bzeta_{j, i}- \btheta_i^{t+1}\right) 
        \nonumber \\ 
        &\hspace{7ex} + \frac{\rho_j}{2}\left\|\bzeta_{j, i}-\btheta_i^{t+1}\right\|_{2}^{2} \nonumber \\
        &\text{s.t.} \, \enspace \left\| 
    \begin{bmatrix}
    2\bm{w}_0 ,\\
    z_0 - 1
    \end{bmatrix}
    \right\|_2
    \leq z_0 + 1 \nonumber ,\\
        & \hspace{4.15ex} \left\| 
    \begin{bmatrix}
    2\bm{w}_k ,\\
    z_k - \bzeta_{j, k  + (i - 1)Q/s}
    \end{bmatrix}
    \right\|_2
    \leq z_k + \bzeta_{j, k  + (i - 1)Q/s}, \nonumber \\
    & \hspace{30.5ex} k \in \{ 1,2, \ldots, Q/s\} \nonumber ,\\
        & \hspace{4ex} \bm{y} = \sum_{k = 1}^{Q/s} \bm{L}_{k + (i - 1)Q/s} \bm{w}_k \nonumber \\
        & \hspace{4ex} + \left( \sum_{n \neq i}\sum_{k = 1}^{Q/s} \bzeta_{j, k  + (n - 1)Q/s}^\eta \bm{K}_{k  + (n - 1)Q/s} + \sigma_e^2 \bm{I}_n \right)^{\frac{1}{2}} \bm{w}_0 \nonumber ,\\
        & \hspace{4.75ex} \bzeta_{j, i} \geq \bm{0}, \enspace \bm{z} \geq \bm{0}.
\end{align}
Notice that the term $\left(\boldsymbol{\lambda}_{j, i}^{t}\right)^{\top}\btheta_i^{t+1}$ is a constant and the 2-norm term in the objective, $\left\|\bzeta_{j, i}-\btheta_i^{t+1}\right\|_{2}^{2}$, can be replaced with a rotated quadratic cone constraint by introducing a variable $v\geq 0$:
\begin{align}
        &\min_{\bzeta_{j, i}, \bm{z}, \bm{w}, v} \, (\bm{1}^\top \bm{z}) + \left(\blambda_{j, i}^t - \nabla_{\bzeta_{j, i}} h(\bzeta_j^\eta)\right)^\top \bzeta_{j, i} +\frac{\rho_j}{2} v \nonumber \\
        &\text{s.t.} \, \enspace \left\| 
    \begin{bmatrix}
    2\bm{w}_0 ,\\
    z_0 - 1
    \end{bmatrix}
    \right\|_2
    \leq z_0 + 1 \nonumber ,\\
        & \hspace{4.15ex} \left\| 
    \begin{bmatrix}
    2\bm{w}_k ,\\
    z_k - \bzeta_{j, k  + (i - 1)Q/s}
    \end{bmatrix}
    \right\|_2
    \leq z_k + \bzeta_{j, k  + (i - 1)Q/s}, \nonumber \\
    & \hspace{30.5ex} k \in \{ 1,2, \ldots, Q/s\} \nonumber ,\\
        & \hspace{4ex} \bm{y} = \sum_{k = 1}^{Q/s} \bm{L}_{k + (i - 1)Q/s} \bm{w}_k \nonumber \\
        & \hspace{4.5ex} + \left( \sum_{n \neq i}\sum_{k = 1}^{Q/s} \bzeta_{j, k  + (n - 1)Q/s}^\eta \bm{K}_{k  + (n - 1)Q/s} + \sigma_e^2 \bm{I}_n \right)^{\frac{1}{2}} \!\bm{w}_0 \nonumber ,\\
        & \hspace{4ex}  2v \frac{1}{2} \geq \sum_{k = 1}^{Q/s} \left(\bzeta_{j, k + (i - 1)Q/s}-\btheta_{k + (i - 1)Q/s}^{t+1}\right)^2 \nonumber ,\\
        & \hspace{4ex} \bzeta_{j, i} \geq \bm{0}, \enspace \bm{z} \geq \bm{0},
\end{align}
which is equivalent to the conic problem given in Eq. (\ref{eq:socp_d2sca_conic}).

\section{Proof of Theorem \ref{theorem:dsca_convergence}} \label{appendix:proof_theorem_dsca_coverg}
\begin{proof}
To prove the convergence of the DSCA algorithm, in general, we need to select an appropriate step size, which is crucial to ensure the algorithm's convergence, as mentioned in \cite{scutariParallelDistributedSuccessive2018}. However, in our case, a step size is no longer necessary as each surrogate function constructed based on Eq. (\ref{eq:dsca-surrogate}) already makes $\ell(\bzeta_{j, i},  \bzeta_j^\eta)$ a global upper bound for the local objective function. In particular, $\ell(\bzeta_{j, i},  \bzeta_j^\eta)$ is a majorization function for the local objective function in the $i$-th block,  $L(\bzeta_{j,i}, \bzeta_{j, -i}^\eta) \triangleq l(\bzeta_{j,i}, \bzeta_{j, -i}^\eta) \!+\! \left(\boldsymbol{\lambda}_{j, i}^{t}\right)^{\top} \!\! \left(\bzeta_{j, i} \!-\! \btheta_i^{t+1}\right) \!+\! \frac{\rho_j}{2}\left\|\bzeta_{j, i}-\btheta_i^{t+1}\right\|_{2}^{2}$, satisfying the following conditions:
\begin{assumption}
	\label{assumption:mm}
	Each $\ell(\bzeta_{j, i},  \bzeta_j^\eta)$ is a majorization function of $L(\bzeta_{j, i}, \bzeta_{j, -i}^\eta)$, for $i \in \{1, \ldots, s\}$, and satisfies the following conditions:
	\begin{enumerate}
	\item $\ell(\bzeta_{j, i}^\eta,  \bzeta_j^\eta) = L(\bzeta_{j, i}^\eta, \bzeta_{j, -i}^\eta)$, for $\bzeta_{j, i}^\eta \in \Theta_i$;
        \item $L(\bzeta_{j, i}, \bzeta_{j, -i}^\eta) \leq \ell(\bzeta_{j, i},  \bzeta_j^\eta)$, $\forall \bzeta_{j, i} \in \Theta_i, \bzeta_j^t \in \Theta$.
	\end{enumerate}
 \end{assumption}
\noindent Consequently, we can deduce that,
\begin{align}
L(\bzeta_{j, i}^{\eta+1}, \bzeta_{j, -i}^\eta) \leq \ell(\bzeta_{j, i}^{\eta+1},  \bzeta_j^\eta) \leq \ell(\bzeta_{j, i}^\eta,  \bzeta_j^\eta) = L(\bzeta_{j, i}^\eta, \bzeta_{j, -i}^\eta),
\end{align}
ensuring the sequence $\{ L(\bzeta_{j, i}^{\eta+1}, \bzeta_{j, -i}^{\eta}) \}_{\eta \in \mathbb{N}}$, is non-increasing. Based upon Assumption \ref{assumption:surrogate-dsca} and \ref{assumption:mm} for each $\ell (\bzeta_{j, i} ,  \bzeta_j^\eta)$, Eq. (\ref{eq:dsca_cvg}) is guaranteed to hold as a result of \cite[Section II.2]{scutariParallelDistributedSuccessive2018}.
\end{proof}

\section{Proof of Lemma \ref{lemma:stochastic-quantization}}
\label{appendix:lemma-proof}
\begin{proof}
By using the definition of stochastic quantization function provided in Eq. (\ref{def:stochastic-quantization}), we have:
\begin{equation}
\begin{aligned}
    \mathbb{E}[\mathcal{Q}(\theta)]
    &= \tau_j \left(1 - \frac{\theta -\tau_j}{\Delta} \right) + \tau_{j + 1}\left(\frac{\theta -\tau_j}{\Delta}\right) \\
    &=\tau_j \left(1 - \frac{\theta -\tau_j}{\Delta} \right) + (\tau_j + \Delta)\left(\frac{\theta -\tau_j}{\Delta} \right)\\
    &= \tau_j - \tau_j\left(\frac{\theta -\tau_j}{\Delta} \right) + \tau_j \left(\frac{\theta -\tau_j}{\Delta} \right) + \theta - \tau_j \\
    &= \theta,
\end{aligned}
\end{equation}
so $\mathcal{Q}(\theta)$ is an unbiased estimator of $\theta$. Next, we can show that
\begin{align}
    &\mathbb{E}\left[(\theta - \mathcal{Q}(\theta))^2\right] \nonumber \\
    &= (\theta -\tau_j)^2 \left(1 - \frac{\theta -\tau_j}{\Delta} \right) + (\theta - \tau_{j + 1})^2 \left(\frac{\theta -\tau_j}{\Delta} \right) \nonumber \\
    &= \Delta^2 \underbrace{\left(1 - \frac{\theta -\tau_j}{\Delta} \right)\left(\frac{\theta -\tau_j}{\Delta} \right)}_{\leq 1/4}\nonumber\\
    &\leq \frac{\Delta^2}{4},
\end{align}
by using the fact that the variance of a random variable that takes value on $[\tau_j, \tau_{j+1})$ is bounded by 1/4 \cite{xiaoDecentralizedEstimationInhomogeneous2005}. This completes the proof for Lemma \ref{lemma:stochastic-quantization}.
\end{proof}

\bibliographystyle{IEEEtran}
\bibliography{references}

\begin{thebibliography}{10}
\providecommand{\url}[1]{#1}
\csname url@samestyle\endcsname
\providecommand{\newblock}{\relax}
\providecommand{\bibinfo}[2]{#2}
\providecommand{\BIBentrySTDinterwordspacing}{\spaceskip=0pt\relax}
\providecommand{\BIBentryALTinterwordstretchfactor}{4}
\providecommand{\BIBentryALTinterwordspacing}{\spaceskip=\fontdimen2\font plus
\BIBentryALTinterwordstretchfactor\fontdimen3\font minus \fontdimen4\font\relax}
\providecommand{\BIBforeignlanguage}[2]{{%
\expandafter\ifx\csname l@#1\endcsname\relax
\typeout{** WARNING: IEEEtran.bst: No hyphenation pattern has been}%
\typeout{** loaded for the language `#1'. Using the pattern for}%
\typeout{** the default language instead.}%
\else
\language=\csname l@#1\endcsname
\fi
#2}}
\providecommand{\BIBdecl}{\relax}
\BIBdecl

\bibitem{suwandiGaussianProcessRegression2022}
R.~C. Suwandi, Z.~Lin, Y.~Sun, Z.~Wang, L.~Cheng, and F.~Yin, ``Gaussian {{process regression}} with {{grid spectral mixture kernel}}: {{Distributed learning}} for {{multidimensional data}},'' in \emph{Proc. Int. Conf. Inf. Fusion (FUSION)}, Linkoping, Sweden, Jul. 2022, pp. 1--8.

\bibitem{chengRethinkingBayesianLearning2022}
L.~Cheng, F.~Yin, S.~Theodoridis, S.~Chatzis, and T.-H. Chang, ``Rethinking {{Bayesian learning}} for {{data analysis}}: {{The}} art of prior and inference in sparsity-aware modeling,'' \emph{IEEE Signal Process. Mag.}, vol.~39, no.~6, pp. 18--52, Nov. 2022.

\bibitem{rasmussenGaussianProcessesMachine2006}
C.~E. Rasmussen and C.~K.~I. Williams, \emph{{G}aussian Processes for Machine Learning}.\hskip 1em plus 0.5em minus 0.4em\relax {MIT Press}, 2006.

\bibitem{pan2017prediction}
Y.~Pan, X.~Yan, E.~A. Theodorou, and B.~Boots, ``Prediction under uncertainty in sparse spectrum {G}aussian processes with applications to filtering and control,'' in \emph{Int. Conf. Mach. Learn. (ICML)}, Sydney, NSW, Australia, Aug. 2017, pp. 2760--2768.

\bibitem{schmidt2023probabilistic}
A.~Schmidt, P.~Morales-{\'A}lvarez, and R.~Molina, ``Probabilistic attention based on {G}aussian processes for deep multiple instance learning,'' \emph{IEEE Trans. Neural Netw. Learn. Syst.}, vol.~35, no.~8, pp. 10\,909--10\,922, Aug. 2024.

\bibitem{skolidis2013semisupervised}
G.~Skolidis and G.~Sanguinetti, ``Semisupervised multitask learning with {G}aussian processes,'' \emph{IEEE Trans. Neural Netw. Learn. Syst.}, vol.~24, no.~12, pp. 2101--2112, Dec. 2013.

\bibitem{lin2023TAG}
K.~Lin, D.~Li, Y.~Li, S.~Chen, Q.~Liu, J.~Gao, Y.~Jin, and L.~Gong, ``{TAG}: Teacher-advice mechanism with {G}aussian process for reinforcement learning,'' \emph{IEEE Trans. Neural Netw. Learn. Syst.}, pp. 1--15, Apr. 2023.

\bibitem{theodoridisMachineLearningBayesian2020}
S.~Theodoridis, \emph{Machine Learning: A {{Bayesian}} and Optimization Perspective}, 2nd~ed.\hskip 1em plus 0.5em minus 0.4em\relax {Academic Press}, 2020.

\bibitem{chowdhary2015BayesianNonparametric}
G.~Chowdhary, H.~A. Kingravi, J.~P. How, and P.~A. Vela, ``{B}ayesian nonparametric adaptive control using {G}aussian processes,'' \emph{IEEE Trans. Neural Netw. Learn. Syst.}, vol.~26, no.~3, pp. 537--550, Mar. 2015.

\bibitem{xuWirelessTrafficPrediction2019}
Y.~Xu, F.~Yin, W.~Xu, J.~Lin, and S.~Cui, ``Wireless {{traffic prediction with scalable Gaussian process}}: {{Framework}}, {{algorithms}}, and {{verification}},'' \emph{IEEE J. Select. Areas Commun.}, vol.~37, no.~6, pp. 1291--1306, Jun. 2019.

\bibitem{zhang2021kernel}
J.~Zhang, H.~Ning, X.~Jing, and T.~Tian, ``Online kernel learning with adaptive bandwidth by optimal control approach,'' \emph{IEEE Trans. Neural Netw. Learn. Syst.}, vol.~32, no.~5, pp. 1920--1934, 2021.

\bibitem{lazaro2010sparse}
M.~L{\'a}zaro-Gredilla, J.~Quinonero-Candela, C.~E. Rasmussen, and A.~R. Figueiras-Vidal, ``Sparse spectrum {G}aussian process regression,'' \emph{J. Mach. Learn. Res.}, vol.~11, pp. 1865--1881, 2010.

\bibitem{cui2024probRL}
Y.~Cui, W.~Shi, H.~Yang, C.~Shao, L.~Peng, and H.~Li, ``Probabilistic model-based reinforcement learning unmanned surface vehicles using local update sparse spectrum approximation,'' \emph{IEEE Trans. Industr. Inform.}, vol.~20, no.~2, pp. 1283--1293, Feb. 2024.

\bibitem{wilsonGaussianProcessKernels2013}
A.~G. Wilson and R.~P. Adams, ``{G}aussian process kernels for pattern discovery and extrapolation,'' in \emph{Proc. Int. Conf. Mach. Learn. (ICML)}, {Atlanta, GA, USA}, Jun. 2013, pp. 1067--1075.

\bibitem{yinLinearMultipleLowRank2020}
F.~Yin, L.~Pan, T.~Chen, S.~Theodoridis, Z.-Q.~T. Luo, and A.~M. Zoubir, ``Linear {{multiple low-rank kernel based stationary {G}aussian processes regression}} for {{time series}},'' \emph{IEEE Trans. Signal Process.}, vol.~68, pp. 5260--5275, Sep. 2020.

\bibitem{chen2022recent}
K.~Chen, Q.~Kong, Y.~Dai, Y.~Xu, F.~Yin, L.~Xu, and S.~Cui, ``Recent advances in data-driven wireless communication using {G}aussian processes: A comprehensive survey,'' \emph{China Commun.}, vol.~19, no.~1, pp. 218--237, 2022.

\bibitem{Liu2020scalableGP}
H.~Liu, Y.-S. Ong, X.~Shen, and J.~Cai, ``When {G}aussian process meets big data: A review of scalable {GP}s,'' \emph{IEEE Trans. Neural Netw. Learn. Syst.}, vol.~31, no.~11, pp. 4405--4423, Nov. 2020.

\bibitem{zhai2023distributed}
P.~Zhai and R.~T. Rajan, ``Distributed {G}aussian process hyperparameter optimization for multi-agent systems,'' in \emph{IEEE Int. Conf. Acoust. Speech Signal Process. (ICASSP)}, Rhodes Island, Greece, May 2023, pp. 1--5.

\bibitem{dang2024GP}
Z.~Dang, B.~Gu, C.~Deng, and H.~Huang, ``Asynchronous parallel large-scale {G}aussian process regression,'' \emph{IEEE Trans. Neural Netw. Learn. Syst.}, vol.~35, no.~6, pp. 8683--8694, 2024.

\bibitem{yinFedLocFederatedLearning2020}
F.~Yin, Z.~Lin, Q.~Kong, Y.~Xu, D.~Li, S.~Theodoridis, and S.~R. Cui, ``{{FedLoc}}: {{Federated learning framework}} for {{data-driven cooperative localization}} and {{location data processing}},'' \emph{IEEE Open J. Signal Process.}, vol.~1, pp. 187--215, Nov. 2020.

\bibitem{titsias2009variational}
M.~Titsias, ``Variational learning of inducing variables in sparse {Gaussian} processes,'' in \emph{Proc. Int. Conf. Artif. Intell. Stat. (AISTATS)}, Clearwater, FL, USA, Apr. 2009, pp. 567--574.

\bibitem{liu2021gp}
H.~Liu, Y.-S. Ong, and J.~Cai, ``Large-scale heteroscedastic regression via {G}aussian process,'' \emph{IEEE Trans. Neural Netw. Learn. Syst.}, vol.~32, no.~2, pp. 708--721, 2021.

\bibitem{boydConvexOptimization2004}
S.~P. Boyd and L.~Vandenberghe, \emph{Convex Optimization}.\hskip 1em plus 0.5em minus 0.4em\relax {Cambridge University Press}, 2004.

\bibitem{scutariParallelDistributedSuccessive2018}
G.~Scutari and Y.~Sun, ``Parallel and {{distributed successive convex approximation methods}} for {{big-data optimization}},'' in \emph{Multi-Agent {{Optim.}}}\hskip 1em plus 0.5em minus 0.4em\relax {Springer}, 2018, pp. 141--308.

\bibitem{mosek}
M.~ApS, \emph{The MOSEK optimization toolbox for MATLAB manual. Version 10.0.}, 2022.

\bibitem{sunMajorizationMinimizationAlgorithmsSignal2017}
Y.~Sun, P.~Babu, and D.~P. Palomar, ``Majorization-{{minimization algorithms}} in {{signal processing}}, {{communications}}, and {{machine learning}},'' \emph{IEEE Trans. Signal Process.}, vol.~65, no.~3, pp. 794--816, Feb. 2017.

\bibitem{hatzinakosGaussianMixturesTheir2001}
K.~N.~P. Hatzinakos, Dimitris, ``Gaussian {{mixtures}} and {{their applications}} to {{signal processing}},'' in \emph{Adv. {{Signal Process. Handbook}}}.\hskip 1em plus 0.5em minus 0.4em\relax {CRC Press}, 2001.

\bibitem{kashyapQuantizedConsensus2007}
A.~Kashyap, T.~Ba{\c s}ar, and R.~Srikant, ``Quantized consensus,'' \emph{Automatica}, vol.~43, no.~7, pp. 1192--1203, Jul. 2007.

\bibitem{amiriFederatedLearningQuantized2020}
M.~M. Amiri, D.~Gunduz, S.~R. Kulkarni, and H.~V. Poor, ``Federated learning with quantized global model updates,'' 2020, \textit{arXiv:2006.10672}.

\bibitem{boydDistributedOptimizationStatistical2011}
S.~Boyd, N.~Parikh, E.~Chu, B.~Peleato, and J.~Eckstein, ``Distributed {{optimization}} and {{statistical learning}} via the {{alternating direction method}} of {{multipliers}},'' \emph{Found. Trends Mach. Learn.}, vol.~3, no.~1, pp. 1--122, Jan. 2011.

\bibitem{zhuQuantizedConsensusADMM2016a}
S.~Zhu and B.~Chen, ``Quantized consensus by the {ADMM}: Probabilistic versus deterministic quantizers,'' \emph{IEEE Trans. Signal Process.}, vol.~64, no.~7, pp. 1700--1713, Apr. 2016.

\bibitem{zhaoParticle2015}
Y.~Zhao, F.~Yin, F.~Gunnarsson, M.~Amirijoo, E.~Özkan, and F.~Gustafsson, ``Particle filtering for positioning based on proximity reports,'' in \emph{{Proc. Int. Conf. Inf. Fusion (FUSION)}}, Washington, DC, USA, Jul. 2015, pp. 1046--1052.

\bibitem{jinDithering2015}
D.~Jin, A.~M. Zoubir, F.~Yin, C.~Fritsche, and F.~Gustafsson, ``Dithering in quantized {RSS} based localization,'' in \emph{IEEE Int. Workshop on Comput. Adv. in Multi-Sens. Adapt. Process. (CAMSAP)}, Cancun, Mexico, Dec. 2015, pp. 245--248.

\bibitem{doan2020fast}
T.~T. Doan, S.~T. Maguluri, and J.~Romberg, ``Fast convergence rates of distributed subgradient methods with adaptive quantization,'' \emph{IEEE Trans. Autom. Control}, vol.~66, no.~5, pp. 2191--2205, 2021.

\bibitem{xiaoDecentralizedEstimationInhomogeneous2005}
J.-J. Xiao and Z.-Q. Luo, ``Decentralized estimation in an inhomogeneous sensing environment,'' \emph{IEEE Trans. Inf. Theory}, vol.~51, no.~10, pp. 3564--3575, Oct. 2005.

\bibitem{cvx}
M.~Grant and S.~Boyd, ``{CVX}: Matlab software for disciplined convex programming, version 2.1,'' Mar. 2014.

\bibitem{hongConvergenceAnalysisAlternating2016}
M.~Hong, Z.-Q. Luo, and M.~Razaviyayn, ``Convergence {{analysis}} of {{alternating direction method}} of {{multipliers}} for a {{family}} of {{nonconvex problems}},'' \emph{SIAM J. Optim.}, vol.~26, no.~1, pp. 337--364, Jan. 2016.

\bibitem{dengGlobalLinearConvergence2016}
W.~Deng and W.~Yin, ``On the {{global}} and {{linear convergence}} of the {{generalized alternating direction method}} of {{multipliers}},'' \emph{J. Sci. Comput.}, vol.~66, no.~3, pp. 889--916, Mar. 2016.

\bibitem{wohlbergADMMPenaltyParameter2017}
B.~Wohlberg, ``{ADMM} penalty parameter selection by residual balancing,'' 2017, \textit{arXiv:1704.06209}.

\bibitem{garciaRobustSmoothingGridded2010}
D.~Garcia, ``Robust smoothing of gridded data in one and higher dimensions with missing values,'' \emph{Comput. Statist. \& Data Anal.}, vol.~54, no.~4, pp. 1167--1178, Apr. 2010.

\bibitem{hochreiterLongShortTermMemory1997}
S.~Hochreiter and J.~Schmidhuber, ``Long {{short-term memory}},'' \emph{Neural Comput.}, vol.~9, no.~8, pp. 1735--1780, Nov. 1997.

\bibitem{zhou2021informer}
H.~Zhou, S.~Zhang, J.~Peng, S.~Zhang, J.~Li, H.~Xiong, and W.~Zhang, ``Informer: Beyond efficient transformer for long sequence time-series forecasting,'' in \emph{Proc. AAAI Conf. Artif. Intell. (AAAI)}, Virtual, Online, Feb. 2021, pp. 11\,106--11\,115.

\bibitem{lin2023segrnn}
S.~Lin, W.~Lin, W.~Wu, F.~Zhao, R.~Mo, and H.~Zhang, ``{SegRNN}: Segment recurrent neural network for long-term time series forecasting,'' 2023, \textit{arXiv:2308.11200}.

\bibitem{nie2023time}
Y.~Nie, N.~H. Nguyen, P.~Sinthong, and J.~Kalagnanam, ``A time series is worth 64 words: Long-term forecasting with transformers,'' in \emph{Proc. Int. Conf. Learn. Represent. (ICLR)}, Kigali, Rwanda, May 2023.

\bibitem{wang2022micn}
H.~Wang, J.~Peng, F.~Huang, J.~Wang, J.~Chen, and Y.~Xiao, ``{MICN}: Multi-scale local and global context modeling for long-term series forecasting,'' in \emph{Proc. Int. Conf. Learn. Represent. (ICLR)}, Kigali, Rwanda, May 2023.

\bibitem{das2024longterm}
A.~Das, W.~Kong, A.~Leach, S.~Mathur, R.~Sen, and R.~Yu, ``Long-term forecasting with {TiDE}: Time-series dense encoder,'' \emph{Trans. Mach. Learn. Res.}, pp. 1--21, Aug. 2023.

\bibitem{hong2023domkl}
S.~Hong and J.~Chae, ``Distributed online learning with multiple kernels,'' \emph{IEEE Trans. Neural Netw. Learn. Syst.}, vol.~34, no.~3, pp. 1263--1277, 2023.

\bibitem{hong2024omkl}
S.~Hong, ``Online multikernel learning method via online biconvex optimization,'' \emph{IEEE Trans. Neural Netw. Learn. Syst.}, vol.~35, no.~11, pp. 16\,630--16\,643, 2024.

\bibitem{wipf2004SBL}
D.~Wipf and B.~Rao, ``Sparse {B}ayesian learning for basis selection,'' \emph{IEEE Trans. Signal Process.}, vol.~52, no.~8, pp. 2153--2164, Aug. 2004.

\bibitem{luenbergerLinearNonlinearProgramming1984}
D.~G. Luenberger, \emph{Linear and {{Nonlinear Programming}}}, 2nd~ed.\hskip 1em plus 0.5em minus 0.4em\relax {Addison-Wesley}, Jan. 1984.

\bibitem{vandenbergheSemidefiniteProgramming1996}
L.~Vandenberghe and S.~Boyd, ``Semidefinite {{programming}},'' \emph{SIAM Rev.}, vol.~38, no.~1, pp. 49--95, Mar. 1996.

\bibitem{loboApplicationsSecondorderCone1998}
M.~S. Lobo, L.~Vandenberghe, S.~Boyd, and H.~Lebret, ``Applications of second-order cone programming,'' \emph{Linear Algebra and its Appl.}, vol. 284, no.~1, pp. 193--228, Nov. 1998.

\bibitem{rahimiRandomFeaturesLargeScale2007}
A.~Rahimi and B.~Recht, ``Random {{features}} for {{large-scale kernel machines}},'' in \emph{Adv. Neural Inf. Process. Syst. (NeurIPS)}, Vancouver, BC, Canada, Dec. 2007, pp. 1177--1184.

\bibitem{williamsUsingNystromMethod2000}
C.~Williams and M.~Seeger, ``Using the {{Nystr\"om method}} to {{speed up kernel machines}},'' in \emph{Adv. Neural Inf. Process. Syst. (NeurIPS)}, Denver, CO, USA, Nov. 2000, pp. 682--688.

\end{thebibliography}

\end{document}